\documentclass{article}
\usepackage[final]{neurips_2021}

\setcitestyle{numbers,square,comma}

\usepackage[utf8]{inputenc}
\usepackage[T1]{fontenc}
\usepackage{hyperref}      
\usepackage{url}           
\usepackage{booktabs}     
\usepackage{amsfonts}
\usepackage{nicefrac}
\usepackage{microtype}
\usepackage{xcolor}
\usepackage{wrapfig}

\usepackage{amssymb}
\usepackage{amsthm}
\usepackage{amsmath}
\usepackage{listings}
\usepackage{enumitem}
\usepackage{hyperref}
\usepackage{graphicx}
\usepackage{subfigure}
\usepackage{algorithmic}
\usepackage{algorithm}
\newtheorem{thm}{Theorem}

\title{Dynamic Bottleneck for Robust Self-Supervised Exploration}

\author{%
  Chenjia Bai \\
  Harbin Institute of Technology\\
  \texttt{bai\_chenjia@stu.hit.edu.cn} \\
  \And
  Lingxiao Wang \\
  Northwestern University\\
  \texttt{lingxiaowang2022@u.northwestern.edu} \\
  \And
  Lei Han \\
  Tencent Robotics X\\
  \texttt{lxhan@tencent.com} \\
  \And
  Animesh Garg \\
  University of Toronto, Vector Institute, NVIDIA\\
  \texttt{garg@cs.toronto.edu} \\
  \And
  Jianye Hao \\
  Tianjin University\\
  \texttt{jianye.hao@tju.edu.cn} \\
  \And
  Peng Liu \\
  Harbin Institute of Technology\\
  \texttt{pengliu@hit.edu.cn} \\
  \And
  Zhaoran Wang \\
  Northwestern University\\
  \texttt{zhaoranwang@gmail.com} \\
}

\begin{document}

\maketitle

\begin{abstract}
Exploration methods based on pseudo-count of transitions or curiosity of dynamics have achieved promising results in solving reinforcement learning with sparse rewards. However, such methods are usually sensitive to environmental dynamics-irrelevant information, e.g., white-noise. To handle such dynamics-irrelevant information, we propose a Dynamic Bottleneck (DB) model, which attains a dynamics-relevant representation based on the information-bottleneck principle. Based on the DB model, we further propose DB-bonus, which encourages the agent to explore state-action pairs with high information gain. We establish theoretical connections between the proposed DB-bonus, the upper confidence bound (UCB) for linear case, and the visiting count for tabular case. We evaluate the proposed method on Atari suits with dynamics-irrelevant noises. Our experiments show that exploration with DB bonus outperforms several state-of-the-art exploration methods in noisy environments.
\end{abstract}

\section{Introduction}

The tradeoff between exploration and exploitation has long been a major challenge in reinforcement learning (RL)~\citep{DQN-2015,sutton-2018,yang2021exploration}. Generally, excessive exploitation of the experience suffers from the potential risk of being suboptimal, whereas excessive exploration of novel states hinders the improvement of the policy. A straightforward way to tackle the exploration-exploitation dilemma is to enhance exploration efficiency while keeping exploitation in pace. When the extrinsic rewards are dense, reward shaping is commonly adopted for efficient exploration. However, in many real-world applications such as autonomous driving \citep{michelmore2018evaluating}, the extrinsic rewards are sparse, making efficient exploration a challenging task in developing practical RL algorithms. The situations become even worse when the extrinsic rewards are entirely unavailable. In such a scenario, the task of collecting informative trajectories from exploration is known as the self-supervised exploration \citep{largescale-2019}.

An effective approach to self-supervised exploration is to design a dense intrinsic reward that motivates the agent to explore novel transitions. Previous attempts include count-based \citep{count-2016} and curiosity-driven \citep{curiosity-2017} explorations. The count-based exploration builds a density model to measure the pseudo-count of state visitation and assign high intrinsic rewards to less frequently visited states. In contrast, the curiosity-driven methods maintain a predictive model of the transitions and encourage the agent to visit transitions with high prediction errors.
However, all these methods becomes unstable when the states are noisy, e.g., containing dynamics-irrelevant information. For example, in autonomous driving tasks, the states captured by the camera may contain irrelevant objects, such as clouds that behave similar to Brownian movement. Hence, if we measure the novelty of states or the curiosity of transitions through raw observed pixels, exploration are likely to be affected by the dynamics of these irrelevant objects.

To encourage the agent to explore the most informative transitions of dynamics, we propose a Dynamic Bottleneck (DB) model, which generates a dynamics-relevant representation $Z_t$ of the current state-action pair $(S_t, A_t)$ through the Information-Bottleneck (IB) principle \citep{inbo-2000}. The goal of training DB model is to acquire dynamics-relevant information and discard dynamics-irrelevant features simultaneously. To this end, we maximize the mutual-information $I(Z_t;S_{t+1})$ between a latent representation $Z_t$ and the next state $S_{t+1}$ through maximizing its lower bound and using contrastive learning. Meanwhile, we minimize the mutual-information
$I([S_t, A_t];Z_t)$ between the state-action pair and the corresponding representation to compress dynamics-irrelevant information. Based on our proposed DB model, we further construct a DB-bonus for exploration. DB-bonus measures the novelty of state-action pairs by their information gain with respect to the representation computed from the DB model. We show that the DB-bonus are closely related to the provably efficient UCB-bonus in linear Markov Decision Processes (MDPs) \citep{bandit-2011} and the visiting count in tabular MDPs \citep{auer2002using,regret-2010}. We further estimate the DB-bonus by the learned dynamics-relevant representation from the DB model. We highlight that exploration based on DB-bonus directly utilize the information gain of the transitions, which filters out dynamics-irrelevant noise. We conduct experiments on the Atari suit with dynamics-irrelevant noise injected. Results demonstrate that our proposed self-supervised exploration with DB-bonus is robust to dynamics-irrelevant noise and outperforms several state-of-the-art exploration methods.

\section{Related Work}

Our work is closely related to previous exploration algorithms that construct intrinsic rewards to quantify the novelty of states and transitions. Several early approaches directly define the pseudo-count by certain statistics to measure the novelty of states~\citep{count-2006,count-2012}; more recent methods utilize density model \citep{count-2016,count-2017} or hash map~\citep{tang-2017,opiq-2020} for state statistics. Nevertheless, these approaches are easily affected by dynamics-irrelevant information such as white-noise. The contingency awareness method \citep{Contingency-2019} addresses such an issue by using an attentive model to locate the agent and computes the pseudo-count based on regions around the agent. However, such an approach could ignore features that are distant from the agent but relevant to the transition dynamics. Another line of research measures the novelty through learning a dynamics model and then use the prediction error to generate an intrinsic reward. These methods are known as the curiosity-driven exploration algorithms. Similar to the pseudo-count based methods, curiosity-driven methods become unstable in the presence of noises, because the prediction model is likely to yield high error for stochastic inputs or targets. Some recent attempts improve the curiosity-driven approach by learning the inverse dynamics \citep{curiosity-2017} and variational dynamics \citep{bai-2020} to define curiosity, or utilizes the prediction error of a random network to construct intrinsic rewards~\citep{RND-2019}. However, without explicitly removing dynamics-irrelevant information, these methods are still vulnerable to noises in practice~\citep{largescale-2019}. 

The entropy-based exploration uses state entropy as the intrinsic reward. VISR \citep{hansen2019fast}, APT \citep{liu2021behavior} and APS \citep{liu2021aps} use unsupervised skill discovery for fast task adaptation. In the unsupervised stage, they use $k$-nearest-neighbor entropy estimator to measure the entropy of state, and then use it as the intrinsic reward. RE3 \citep{RE3-2021} and ProtoRL \citep{Proto-2021} use random encoder and prototypes to learn the representation and use state-entropy as bonuses in exploration. Nevertheless, the state entropy will increase significantly if we inject noises in the state space. The entropy-based exploration will be misled by the noises. Previous approaches also quantify the epistemic uncertainty of dynamics through Bayesian network \citep{vime-2016}, bootstrapped $Q$-functions \citep{bootstrap-2016,bai2021principled}, ensemble dynamics \citep{disagree-2019}, and Stein variational inference \citep{implicit-2020} to tackle noisy environments. However, they typically require either complicated optimization methods or large networks. In contrast, DB learns a dynamics-relevant representation and encourages exploration by directly accessing the information gain of new transitions via DB-bonus.

Another closely related line of studies uses the mutual information to promote exploration in RL. Novelty Search (NS) \citep{tao2020novelty} proposes to learn a representation through IB. Curiosity Bottleneck (CB)~\citep{CB-2019} also performs exploration based on IB by measuring the task-relevant novelty. However, both NS and CB require extrinsic rewards to learn a value function and are not applicable for self-supervised exploration. Moreover, NS contains additional $k$-nearest-neighbor to generate intrinsic reward and representation loss to constrain the distance of consecutive states, which are costly for computation. In contrast, our DB model handles self-supervised exploration without accessing extrinsic rewards. EMI \citep{emi-2019} learns a representation by maximizing the mutual information in the forward dynamics and the inverse dynamics , which is different from the IB principle used in our method. In addition, we aim to perform robust exploration to overcome the white-noise problem, while EMI does not have an explicit mechanism to address the noise. 

Our work is also related to representation learning in RL. DrQ \citep{DrQ-2020}, RAD \citep{RAD-2020} and CURL \citep{curl-2020} learn the state representation by data augmentation and contrastive learning \citep{simclr-2020,moco-2020,cl-2018} to improve the data-efficiency of DRL. Deep InfoMax \citep{infomax-2020} and Self-Predictive Representation (SPR) \citep{momentum-2020} learn the contrastive and predictive representations of dynamics, respectively, and utilize such representations as auxiliary losses for policy optimization. However, none of these existing approaches extracts information that benefits exploration. In contrast, we show that the dynamics-relevant representation learned by DB can be utilized for efficient exploration.

\section{The Dynamic Bottleneck}

In this section, we introduce the objective function and architecture of the DB model. We consider an MDP that can be described by a tuple $(\mathcal{O}, \mathcal{A}, \mathbb{P}, r, \gamma)$, which consists of the observation space $\mathcal{O}$, the action space $\mathcal{A}$, the transition dynamics $\mathbb{P}$, the reward function $r$, and the discount factor $\gamma\in(0, 1)$. At each time step, an agent decides to perform an action $a_t\in \mathcal{A}$ after observing $o_t\in \mathcal{O}$, and then the observation transits to $o_{t+1}$ with a reward $r_t$ received. In this paper, we use upper letters, such as $O_t$, to denote random variables and the corresponding lower case letter, such as $o_t$, to represent their corresponding realizations.

We first briefly introduce the IB principle \citep{inbo-2000}. In supervised setting that aims to learn a representation $Z$ of a given input source $X$ with the target source $Y$, IB maximizes the mutual information between $Z$ and $Y$ (i.e. $\max I(Z; Y)$) and restricts the complexity of $Z$ by using the constrain as $I(Z;X)<I_c$. Combining the two terms, the objective of IB is equal to $\max I(Z;Y)-\alpha I(Z;X)$ with the introduction of a Lagrange multiplier. 

DB follows the IB principle \citep{inbo-2000} to learn dynamics-relevant representation. The input variable of the DB model is a tuple $(O_t, A_t)$ that contains the current observation and action, and the target is the next observation $O_{t+1}$. We denote by $S_t$ and $S_{t+1}$ the encoding of observations $O_t$ and $O_{t+1}$. The goal of the DB model is to obtain a compressed latent representation $Z_t$ of $(S_t,A_t)$, that preserves the information that is relevant to $S_{t+1}$ only. Specifically, we use $f^S_o$ and $f^S_m$ as the encoders of two consecutive observations $o_t$ and $o_{t+1}$, respectively. We parameterize the dynamics-relevant representation $z_t$ by a Gaussian distribution with parameter $\phi$, and it takes ($s_t,a_t$) as input. We summarize the DB model as follows,
\begin{equation}\label{eq:paraf}
s_t=f^S_o(o_t;\theta_o),\:\:\:s_{t+1}=f^S_m(o_{t+1};\theta_m),\:\:\:z_t\sim g^Z(s_t,a_t;\phi).
\end{equation}

Following the IB principle, the objective of the DB model seeks to maximize the mutual information $I(Z_t;S_{t+1})$ while minimizing the mutual information $I([S_t,A_t];Z_t)$. To this end, we propose the DB objective by following the IB Lagrangian~\citep{inbo-2000}, which takes the form of
\begin{equation}
\label{eq:IB_lag}
\min -I(Z_t;S_{t+1})+\alpha_1 I([S_t,A_t];Z_t).
\end{equation}
Here $\alpha_1$ is a Lagrange multiplier that quantifies the amount of information about the next state preserved in $Z_t$. Fig. \ref{fig:overview} illustrates the DB objective. We minimize $I([S_t, A_t]; Z_t)$ and consider it as a regularizer in the representation learning. Then the representation learning is done by maximizing the mutual information $I(Z_t, S_{t+1})$. Maximizing $I(Z_t, S_{t+1})$ ensures that we do not discard useful information from $(S_t, A_t)$. In DB, the mutual information is estimated by several variational bounds parameterized by neural networks to enable differentiable and tractable computations. In what follows, we propose a lower bound of \eqref{eq:IB_lag}, which we optimize to train the DB model.

\begin{figure}[t]
\centering
\includegraphics[width=2.0in]{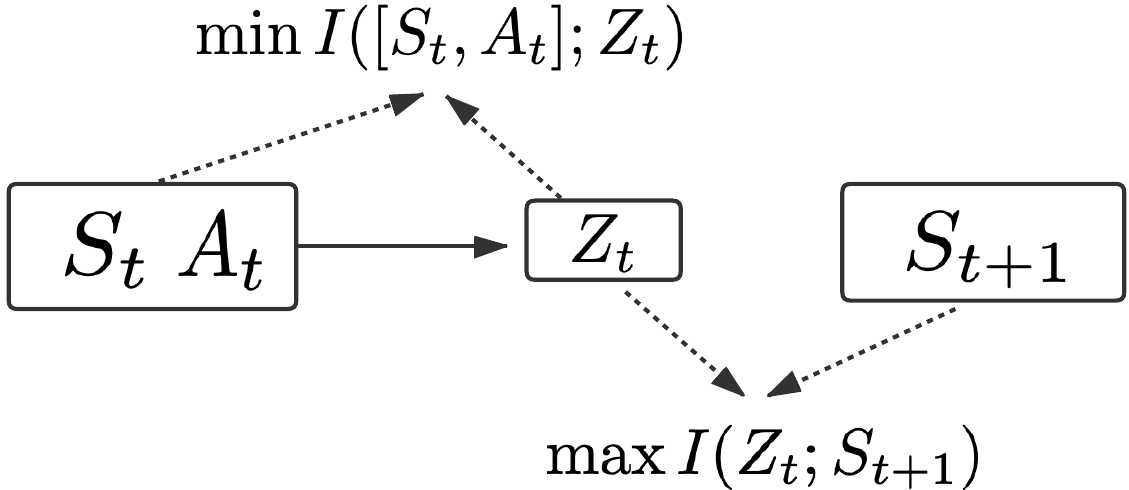}
\caption{Illustration of DB objective. DB minimizes the mutual information $I([S_t,A_t];Z_t)$ to obtain a compressive representation, and maximize $I(Z_t;S_{t+1})$ to preserve the  information.}
\label{fig:overview}
\end{figure}

\subsection{Maximizing the lower bound of $I(Z_t;S_{t+1})$}
\label{sec::lower_bound_0}

As directly maximizing $I(Z_t;S_{t+1})$ is intractable, we propose to optimize a predictive objective, which is a lower bound of $I(Z_t;S_{t+1})$ \citep{dvib-2017}. It holds that
\begin{equation}
\begin{aligned}
I(Z_t;S_{t+1})&=\mathbb{E}_{p(z_t,s_{t+1})}\bigg[\log \frac{p(s_{t+1}|z_t)}{p(s_{t+1})}\bigg]\\&=\mathbb{E}\bigg[\log\frac{q(s_{t+1}|z_t;\psi)}{p(s_{t+1})}\bigg]+D_{\rm KL}[p(s_{t+1}|z_t)\|q(s_{t+1}|z_t;\psi)],	
\end{aligned}
\end{equation}
where $p(s_{t+1}|z_t)$ is an intractable conditional distribution and $q(s_{t+1}|z_t; \psi)$ is a tractable variational decoder with parameter $\psi$. By the non-negativity of the KL-divergence, we obtain the following lower bound,
\begin{align}\label{eq::mutual_info_lower_bound}\nonumber
I(Z_t;S_{t+1})&\geq \mathbb{E}_{p(z_t,s_{t+1})}[\log q(s_{t+1}|z_t;\psi)]+\mathcal{H}(S_{t+1}),
\end{align}
where $\mathcal{H}(\cdot)$ is the entropy. Since $\mathcal{H}(S_{t+1})$ is irrelevant to the parameter $\psi$, maximizing $I(Z_t;S_{t+1})$ is equivalent to maximizing the following lower bound,
\begin{equation}\label{eq::predictive_lower_bound}
I_{\rm pred}\triangleq \mathbb{E}_{p(z_t,s_{t+1})}[\log q(s_{t+1}|z_t;\psi)].
\end{equation}
$I_{\rm pred}$ can be interpreted as the log-likelihood of the next-state encoding $s_{t+1}$ given the dynamics-relevant representation $z_t$. In practice, we parameterize the prediction head $q(s_{t+1}|z_t;\psi)$ by a neural network that outputs a diagonal Gaussian random variable. Since $s_t$, $s_{t+1}$ and $z_t$ are all low-dimensional vectors instead of raw image pixels, optimizing $I_{\rm pred}$ is computationally efficient.

\paragraph{Momentum Encoder} To encode the consecutive observations $o_t$ and $o_{t+1}$, we adopt the Siamese architecture \citep{siamese-1993} that uses the same neural network structures for the two encoders. Nevertheless, we observe that if we train both the encoders by directly maximizing $I_{\rm pred}$, the Siamese architecture tends to converge to a collapsed solution. That is, the generated encodings appears to be uninformative constants. A simple fact is that if both the encoders generate zero vectors, predicting zeros conditioning on $z_t$ (or any variables) is a trivial solution. To address such issue, we update the parameter $\theta_o$ of $f^S_o$ in \eqref{eq:paraf} by directly optimizing $I_{\rm pred}$. Meanwhile, we update the parameter $\theta_m$ of $f^S_m$ by a momentum moving average of $\theta_o$, which takes the form of $\theta_m\leftarrow \tau \theta_m+(1-\tau)\theta_o$. In the sequel, we call $f^S_o$ the online encoder and $f^S_m$ the \emph{momentum} encoder, respectively. Similar techniques is also adopted in 
previous study~\citep{BYOL-2020,momentum-2020} to avoid the mode collapse.

\subsection{Contrastive Objective for Maximizing $I(Z_t;S_{t+1})$}

In addition to the lower bound of $I(Z_t;S_{t+1})$ in \S\ref{sec::lower_bound_0}, we also investigate the approach of maximizing the mutual information by contrastive learning (CL)~\citep{cl-2018}. CL classifies positive samples and negative samples in the learned representation space. An advantage of adopting CL is that training with negative samples plays the role of regularizer, which avoids collapsed solutions. Moreover, the contrastive objective yields a variational lower bound of the mutual information $I(Z;S_{t+1})$. To see such a fact, note that by the Bayes rule, we have
\begin{equation}\label{eq:nce}
I(Z_t;S_{t+1})\geq\mathbb{E}_{p(z_t,s_{t+1})}\mathbb{E}_{S^-}\bigg[\log \frac{\exp (h(z_t,s_{t+1}))}{\sum_{s_j\in S^-\cup{s_{t+1}}}\exp (h(z_t,s_j))}\bigg] \triangleq I_{\rm nce}.
\end{equation}
Here $h$ is a score function which assigns high scores to positive pairs and low score to negative pairs. We refer to Appendix~\ref{app:nce} for a detailed proof of \eqref{eq:nce}. The right-hand side of \eqref{eq:nce} is known as the InfoNCE objective \citep{cl-2018}. The positive samples are obtained by directly sampling the transitions $(s, a, s')$. In contrast, the negative samples are obtained by first sampling a state-action pair $(s, a)$, and then sampling a state $\tilde s$ independently. Then a negative sample is obtained by concatenating them together to form a tuple $(s, a, \tilde s)$. The negative samples do not follow the transition dynamics. In practice, we collect the negative sample by sampling observation encodings randomly from the batch. We remark that comparing with methods that require data augmentation to construct negative samples \citep{moco-2020,curl-2020}, DB utilizes a simple scheme to obtain positive and negative samples from on-policy experiences.


In \eqref{eq:nce}, we adopt the standard bilinear function as the score function $h$, 
which is defined as follows,
\begin{equation}\label{eq:nceh}
h(z_t,s_{t+1})=f^P_o(\bar{q}(z_t;\psi))^{\top} \mathcal{W} f^P_m(s_{t+1}),
\end{equation}
where $f^P_o(\cdot;\varphi_o)$ and $f^P_m(\cdot;\varphi_m)$ project $s_{t+1}$ and the mean value of next-state prediction $q(s_{t+1}|z_t;\varphi)$, i.e., $\bar{q}(\cdot; \psi)$, to a latent space to apply the contrastive loss $I_{\rm nce}$ in \eqref{eq:nce}, and $\mathcal{W}$ is the parameter of the score function. Similar to the observation encoder and MoCo-based architectures \citep{curl-2020,moco-2020,moco2-2020}, we also adopt an online projector $f^P_o$ and a momentum projector $f^P_m$ for $z_t$ and $s_{t+1}$, respectively. The momentum projector is updated by $\varphi_m\leftarrow \tau \varphi_m+(1-\tau)\varphi_o$. 

\subsection{Minimizing the Upper Bound of $I([S_t,A_t];Z_t)$}

We minimize the mutual information $I([S_t,A_t];Z_t)$ through minimizing a tractable upper bound of the mutual information. To this end, we introduce a variational approximation $q(z_t)$ to the intractable marginal $p(z_t)=\int p(s_t,a_t)p(z_t|s_t,a_t)ds_t a_t$. Specifically, the following upper-bound of $I([S_t,A_t];Z_t)$ holds, 
\begin{equation}
\begin{split}
I([S_t,A_t];Z_t)=\mathbb{E}_{p(s_t,a_t)}\Big[\frac{p(z_t|s_t,a_t)}{p(z_t)}&\Big]=\mathbb{E}_{p(s_t,a_t)}\Big[\frac{p(z_t|s_t,a_t)}{q(z_t)}\Big]-D_{\rm KL}\big[p(z_t)\|q(z_t)\big]\\
&\leq \mathbb{E}_{p(s_t,a_t)}\big[D_{\rm KL}[p(z_t|s_t,a_t)\|q(z_t)]\big]\triangleq I_{\rm upper},
\label{eq:upperbound}
\end{split}
\end{equation}
where the inequality follows from the non-negativity of the KL divergence, and $q(z_t)$ is an approximation of the marginal distribution of $Z_t$. We follow \citet{dvib-2017} and use a standard spherical Gaussian distribution $q(z_t)=\mathcal{N}(0,\mathbf{I})$ as the approximation. The expectation of $I_{\rm upper}$ is estimated by sampling from on-policy experiences.

\begin{figure*}[t]
\centering
\includegraphics[width=5.4in]{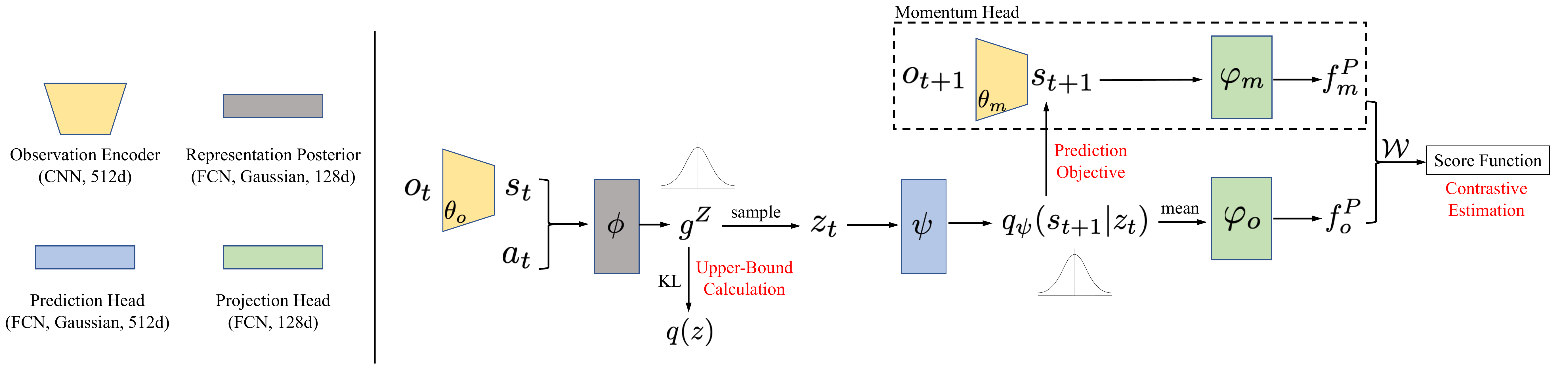}
\caption{The network architecture of DB. The architecture contains several convolution neural networks (CNNs) and fully-connected networks (FCNs). It consists of four components, including (i) the observation encoder $(f^S_o, f^S_m)$, which consists of an online network $f^S_o$ and a momentum network $f^S_m$; (ii) the representation posterior $g^Z(s_t,a_t;\phi)$, which generates a Gaussian distribution for the input state action pair $(s_t, a_t)$; (iii) the prediction head $q_\psi(s_{t+1}|z_t)$, which predicts the next state $s_{t+1}$ based on sampling representation $z_t$ from the representation posterior; and (iv) the projection heads $(f^P_o, f^P_m)$, which maps the next observation encoding $s_{t+1}$ and its prediction from the prediction head to low-dimensional space to perform contrastive estimation. } 
\label{fig:db-network}
\end{figure*}

\subsection{The Loss Function and Architecture}
\label{sec::architecture}

The final loss for training the DB model is a combination of the upper and lower bounds established in previous sections,
\begin{equation}\label{eq:db-loss}
\min_{\theta_o,\phi,\psi,\varphi_o,\mathcal{W}} \mathcal{L_{\rm DB}}= \alpha_1 I_{\rm upper}-\alpha_2 I_{\rm pred}-\alpha_3 I_{\rm nce},
\end{equation}
where $\alpha_1$, $\alpha_2$ and $\alpha_3$ are hyper-parameters. 
As we show in the ablation study (\S\ref{sec::experiment}), all of the three components in the loss plays an important role in learning dynamics-relevant representations. We illustrate the architecture of the DB model in Fig.~\ref{fig:db-network}.
In practice, we minimize $\mathcal{L}_{\rm DB}$ in \eqref{eq:db-loss} by gradient descent, which iteratively updates the parameters of  $f^S_o,g^Z,q_\psi,\mathcal{W}$ and $f^P_o$. Meanwhile, we adopt exponential moving average to update the parameters of $f^S_m$ and $f^P_m$ to avoid collapsed solutions. We refer to Appendix~\ref{app:db-train} for the pseudocode of training DB model.

\section{Exploration with DB-Bonus}

We are now ready to introduce the DB-bonus $r^{\rm db}$ for exploration. In this section, we first present the DB-bonus for self-supervised exploration. We establish the theoretical connections between the DB-bonus and provably efficient bonus functions. We further present the empirical estimation of DB-bonus and the policy optimization algorithm that utilizes the DB-bonus. 

In the sequel, we assume that the learned parameter $\Theta$ of the DB model follows a Bayesian posterior distribution given the training dataset $\mathcal{D}_m = \{(s^i_t, a^i_t, s^i_{t+1})\}_{i \in [0,m]}$, which is a collection of past experiences from $m$ episodes performed by the agent to train the DB model. We aim to estimate the following conceptual reward, which is defined by the mutual information between the parameter of the DB model and the transition dynamics given the training dataset,
\begin{equation}\label{eq:dbreward}
\begin{split}
r^{\rm db}(s_t, a_t) &\triangleq I\bigl(\Theta; (s_t, a_t, S_{t+1}) | \mathcal{D}_m\bigr)^{\nicefrac{1}{2}} \\
&= \Bigl[\mathcal{H}\bigl((s_t, a_t, S_{t+1})|\mathcal{D}_m\bigr)-\mathcal{H}\bigl((s_t, a_t, S_{t+1})|\Theta, \mathcal{D}_m\bigr)\Bigr]^{\nicefrac{1}{2}}.
\end{split}
\end{equation}
Intuitively, DB-bonus defined in \eqref{eq:dbreward} encourages the agent to explore transitions that are maximally informative to the improvement of the DB model.

\subsection{Theoretical Analysis}

We show that the DB-bonus defined in \eqref{eq:dbreward} enjoys well theoretical properties, and establish theoretical connections between $r^{\rm db}$ and
bonuses based on the 
optimism in the face of uncertainty \citep{auer-2007,jin2018q}, which incorporates UCB into value functions in both tabular \citep{minmax-2017,regret-2010,dann-2015} and linear MDPs \citep{jin-2019,oppo-2020}.

\paragraph{Connection to UCB-bonus in linear MDPs}

In linear MDPs, the transition kernel and reward function are assumed to be linear. In such a setting, LSVI-UCB~\citep{jin-2019} provably attains a near-optimal worst-case regret, and
we refer to Appendix~\ref{app:optimistic-lsvi} for the details. The idea of LSVI-UCB is using an optimistic $Q$-value, which is obtained by adding an UCB-bonus $r^{\rm ucb}$ \citep{bandit-2011} to the estimation of the $Q$-value. The UCB-bonus is defined as
$r^{\rm ucb}_t=\beta\cdot\big[\eta(s_t,a_t)^\top\Lambda_t^{-1}\eta(s_t,a_t)\big]^{\nicefrac{1}{2}}$,
where $\beta$ is a constant, $\Lambda_t=\sum_{i=0}^{m}\eta(x_t^{i},a_t^{i})\eta(x_t^{i},a_t^{i})^\top+\lambda \cdot \mathrm{\mathbf{I}}$ is the Gram matrix, and $m$ is the index of the current episode. The UCB-bonus measures the epistemic uncertainty of the state-action and is provably efficient \citep{jin-2019}. 

For linear MDPs, we consider representation $z\in \mathbb{R}^{c}$ as the mean of the posterior $g^Z$ from the DB model, and set $z_t$ to be a linear function of the state-action encoding, i.e., $z_t = W_t\eta(s_t,a_t)$ parameterized by $W_t\in \mathbb{R}^{c\times d}$. Then, the following theorem establishes a connection between the DB-bonus $r^{\rm db}$ and the UCB-bonus $r^{\rm ucb}$.
\begin{thm}
\label{thm:ucb-informal}
In linear MDPs, for tuning parameter $\beta_0>0$, it holds that 
\begin{equation}
\nicefrac{\beta_0}{\sqrt{2}}\cdot r^{\rm ucb}_t \leq I(W_t;(s_t,a_t,S_{t+1})|\mathcal{D}_m)^{\nicefrac{1}{2}} \leq \beta_0\cdot r^{\rm ucb}_t,
\end{equation}
where $I(W_t;(s_t,a_t,S_{t+1})|\mathcal{D}_m)^{\nicefrac{1}{2}}$ is the DB-bonus $r^{\rm db}(s_t,a_t)$ under the linear MDP setting.
\end{thm}

In addition, using $r^{\rm db}$ as bonus leads to the same regret as LSVI-UCB by following a similar proof to \citet{jin-2019}. We refer to Appendix \ref{app:proof-ucb} for the problem setup and the detailed proofs. We remark that Theorem 1 is an approximate derivation because we only consider the predictive objective $I_{\rm pred}$ in \eqref{eq:db-loss} in Theorem 1. Nevertheless, introducing the contrastive objective $I_{\rm nce}$ is important in the training of the DB model as it prevents the mode collapse issue. Theorem \ref{thm:ucb-informal} shows that the DB-bonus provides an instantiation of the UCB-bonus in DRL, which enables us to measure the epistemic uncertainty of high-dimensional states and actions without the linear MDP assumption. 

\paragraph{Connection to visiting count in tabular MDP}
The following theorem establishes connections between DB-bonus and the count-based bonus $r^{\rm count}(s_t,a_t)=\frac{\beta}{\sqrt{N_{s_t,a_t}+\lambda}}$ in tabular MDPs.
\begin{thm}
\label{thm:count-informal}
In tabular MDPs, it holds for the DB-bonus $r^{\rm db}(s_t,a_t)$ and the count-based intrinsic reward $r^{\rm count}(s_t,a_t)$ that,
\begin{equation}
r^{\rm db}(s_t,a_t)\approx \:\frac{\sqrt{|\mathcal{S}|/2}}{\sqrt{N_{s_t,a_t}+\lambda}}\:\:=\:\:\beta_0\cdot r^{\rm count}(s_t,a_t),
\end{equation}
when $N_{s_t,a_t}$ is large, where $\lambda > 0$ is a tuning parameter, $|\mathcal{S}|$ is the number of states in tabular setting.
\end{thm}

We refer to Appendix~\ref{app:proof-count} for a detailed proofs. As a result, DB-bonus can also be considered as a count-based intrinsic reward in the space of dynamics-relevant representations.
\subsection{Empirical Estimation}

To estimate such a bonus under our DB model, we face several challenges. (i) Firstly, estimating the bonus defined in \eqref{eq:dbreward} requires us to parameterize representation under a Bayesian learning framework, whereas our DB model is parameterized by non-Bayesian neural networks. (ii) Secondly, estimating the DB-bonus defined in \eqref{eq:dbreward} requires us to compute the mutual information between the unknown transitions and the estimated model, which is in general hard as we do not have access to such transitions in general. To address such challenges, we estimate a lower bound of the DB-bonus, which is easily implementable and achieves reasonable performance empirically. Specifically, we consider to use $r^{\rm db}_l(s_t,a_t)$ as the lower bound of the information gain in \eqref{eq:dbreward},
\begin{equation}\label{eq:qqq}
\begin{split}
r^{\rm db}(s_t, a_t) \geq \Bigl[\mathcal{H}\bigl(g(s_t, a_t, S_{t+1})|\mathcal{D}_m\bigr)-\mathcal{H}\bigl(g(s_t, a_t, S_{t+1})|\Theta, \mathcal{D}_m\bigr)\Bigr]^{\nicefrac{1}{2}}\triangleq r^{\rm db}_l(s_t,a_t),
\end{split}
\end{equation}
which holds for any mapping $g$ according to Data Processing Inequality (DPI). DPI is an information theoretic concept that can be understood as `post-processing' cannot increase information. Since $g(s_t,a_t,S_{t+1})$ is a post-processing of $(s_t,a_t,S_{t+1})$, we have $I(\Theta; (s_t,a_t,S_{t+1}))>I(\Theta; g(s_t,a_t,S_{t+1}))$, where $g$ is a neural network in practice. In our model, we adopt the following mapping,
\begin{equation}
g(s_t, a_t, S_{t+1}) | \Theta, \mathcal{D}_m = g^Z(s_t, a_t; \phi),
\end{equation}
where $g^Z$ is the representation distribution of DB, and $\phi$ constitutes a part of parameters of the total parameters $\Theta$. Intuitively, since $g^Z$ is trained by IB principle to capture information of transitions, adopting the mapping $g^Z$ to \eqref{eq:qqq}
yields a reasonable approximation of the DB-bonus. It further holds 
\begin{equation}\label{eq::DB_approx_11}
r^{\rm db}_l(s_t, a_t) = \Bigl[\mathcal{H}\bigl(g^{\rm margin}\bigr)-\mathcal{H}\bigl(g^Z(s_t, a_t; \phi)\bigr)\Bigr]^{\nicefrac{1}{2}}
= \mathbb{E}_{\Theta}D_{\rm KL}\big[g^Z(z_t|s_t,a_t; \phi)\|g^{\rm margin}\big]^{\nicefrac{1}{2}},
\end{equation}
where we define $g^{\rm margin}=g(s_t, a_t, S_{t+1}) | \mathcal{D}_m$ as the marginal of the encodings over the posterior of the parameters $\Theta$ of the DB model. In practice, since $g^{\rm margin}$ is intractable, we approximate $g^{\rm margin}$ with standard Gaussian distribution. We remark that such approximation is motivated by the training of DB model, which drives the marginal of representation $g^Z$ toward $\mathcal{N}(0,\mathbf{I})$ through minimizing $I_{\rm upper}$ in \eqref{eq:upperbound}. Such approximation leads to a tractable estimation and stable empirical performances.

In addition, since we do not train the DB model with Bayesian approach, we replace the expectation over posterior $\Theta$ in \eqref{eq::DB_approx_11} by the corresponding point estimation, namely the parameter $\Theta$ of the neural networks trained with DB model on the dataset $\mathcal{D}_m$. To summarize, we utilize the following approximation of the DB-bonus $r^{\rm db}$ proposed in \eqref{eq:dbreward},
\begin{equation}\label{eq:reward-practice}
\hat{r}^{\rm db}_l(s_t,a_t)=D_{\rm KL} \big[g^Z(\cdot|s_t, a_t; \phi)\:\|\:\mathcal{N}(0,\mathbf{I})\big]^{\nicefrac{1}{2}}\approx r^{\rm db}_l(s_t,a_t).
\end{equation}
Since DB is trained by IB principle, which filters out the dynamics-irrelevant information, utilizing the bonus defined in \eqref{eq:reward-practice} allows the agent to conduct robust exploration in noisy environments.

\begin{algorithm}[t]
\caption{SSE-DB}
\label{alg:DB-PPO}
\begin{algorithmic}[1]
\STATE {\bf Initialize:} The DB model and the actor-critic network
\FOR {episode $i = 1$ to $M$}
\FOR {timestep $i = 0$ to $T-1$}
	\STATE Obtain action from the actor $a_t=\pi(s_t)$, then execute $a_t$ and observe the state $s_{t+1}$;
	\STATE Add $(s_t, a_t, s_{t+1})$ into the on-policy experiences;
	\STATE Obtain the DB-bonus $\hat{r}^{\rm db}_l$ of $(s_t, a_t)$ by \eqref{eq:reward-practice};
\ENDFOR
\STATE Update the actor and critic by PPO with the collected on-policy experiences as the input;
\STATE Update DB by gradient descent based on \eqref{eq:db-loss} with the collected on-policy experiences;
\ENDFOR
\end{algorithmic}
\end{algorithm}

We summarize the the overall RL algorithm with self-supervised exploration induced by the DB-bonus in Algorithm~\ref{alg:DB-PPO}, which we refer to as Self-Supervised Exploration with DB-bonus (SSE-DB). For the RL implementation, we adopt Proximal Policy Optimization (PPO) \citep{ppo-2017} with generalized advantage estimation \citep{gae-2016} and the normalization schemes from \citet{largescale-2019}. We refer to Appendix~\ref{app:db-detail} for the implementation details. The codes are available at \url{https://github.com/Baichenjia/DB}.

\section{Experiments}
\label{sec::experiment}
We evaluate SSE-DB on Atari games. We conduct experiments to compare the following methods. (\romannumeral 1)~\textbf{SSE-DB}. The proposed method in Alg.~\ref{alg:DB-PPO}.
(\romannumeral 2) \textbf{Intrinsic Curiosity Model (ICM)} \citep{curiosity-2017}.  ICM uses an inverse dynamics model to extract features related to the actions. ICM further adopts the prediction error of dynamics as the intrinsic reward for exploration.
(\romannumeral 3) \textbf{Disagreement} \citep{disagree-2019}. This method captures epistemic uncertainty by the disagreement among predictions from an ensemble of dynamics models. Disagreement performs competitive to ICM and RND \citep{RND-2019}. Also, this method is robust to white-noise. 
(\romannumeral 4) \textbf{Curiosity Bottleneck (CB)} \citep{CB-2019}. CB quantifies the compressiveness of observation with respect to the representation as the bonus. CB is originally proposed for exploration with extrinsic rewards. We adapt CB for self-supervised exploration by setting the extrinsic reward~zero. 
We compare the model complexity of all the methods in Appendix~\ref{app:db-detail}. Other methods including Novelty Search \citep{tao2020novelty} and Contingency-aware exploration \citep{Contingency-2019} are also deserve to compare. However, we find Novelty Search ineffective in our implementation since the detailed hyper-parameters and empirical results in Atari are not available. Contingency-aware exploration is related to DB while the attention module is relatively complicated and the code is not achievable.

\begin{figure*}[t]
\centering
\includegraphics[width=5.5in]{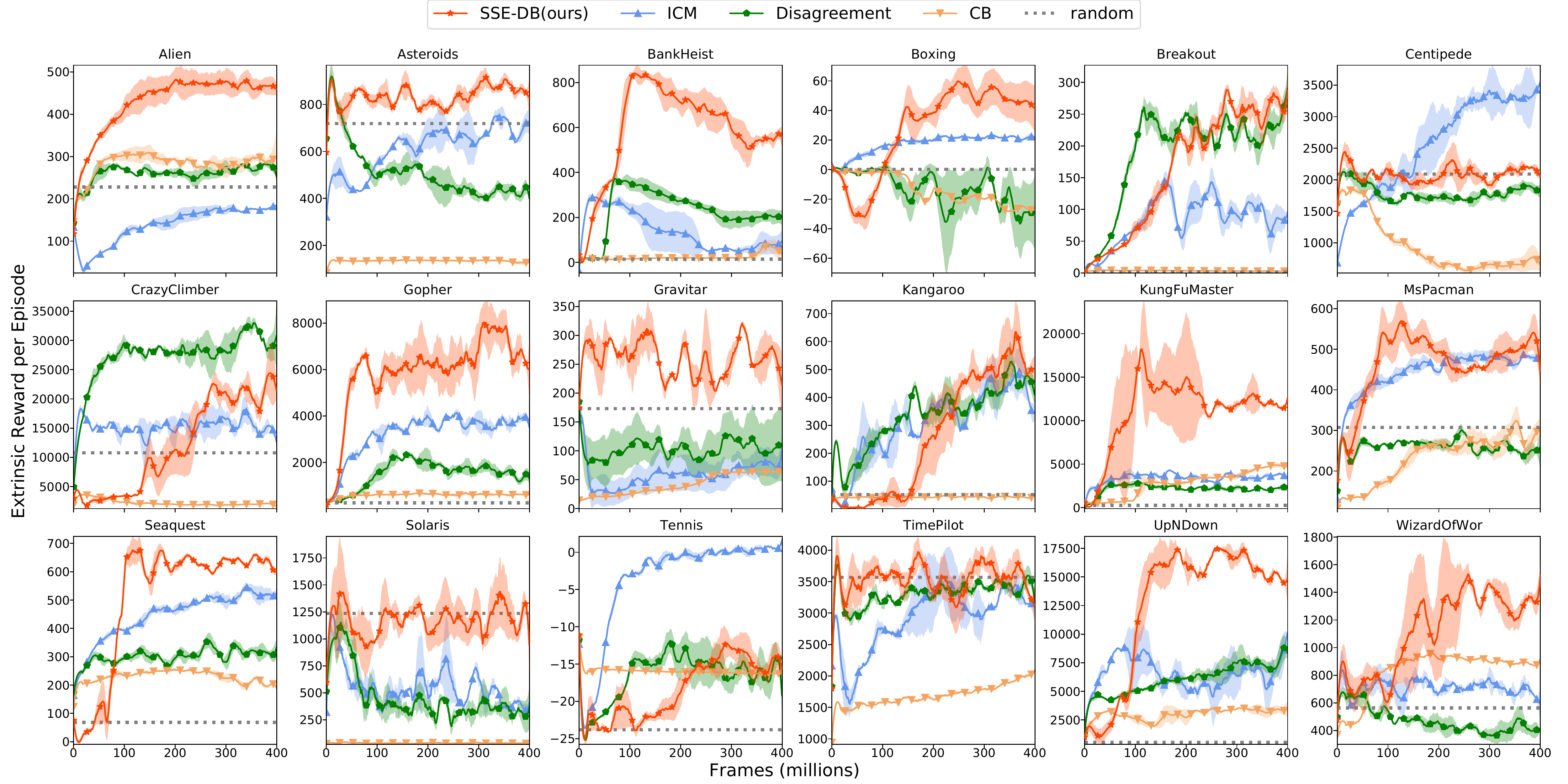}
\caption{The evaluation curve in Atari games. The different methods are trained with different intrinsic rewards. The extrinsic rewards are only used to measure the performance.
Each method was run with three random seeds.}
\label{fig:result-atari}
\end{figure*}

\subsection{The Main Results}

We evaluate all methods on Atari games with high-dimensional observations. The selected 18 games are frequently used in previous approaches for efficient exploration. The overall results are provided in Fig.~\ref{fig:result-atari}. We highlight that in our experiments, the agents are trained without accessing the extrinsic rewards. The extrinsic rewards are ONLY utilized to evaluate the performance of the policies obtained from self-supervised exploration. Our experiments show that SSE-DB performs the best in 15 of 18 tasks, suggesting that dynamics-relevant feature together with DB-bonus helps the exploration of states with high extrinsic rewards.

In addition, since pure exploration without extrinsic rewards is very difficult in most tasks, a random baseline is required to show whether the exploration methods learn meaningful behaviors. We adopt the random score from DQN \citep{DQN-2015} and show the comparison in the figure. In Solaris, Centipede and TimePilot, our method obtains similar scores to random policy, which suggests that relying solely on intrinsic rewards is insufficient to solve these tasks. We also observe that SSE-DB is suboptimal in Tennis. A possible explanation is that for Tennis, the prediction error based methods, such as ICM, could capture additional information in the intrinsic rewards. For example, in Tennis, the prediction error becomes higher when the ball moves faster or when the agent hits the ball towards a tricky direction. The prediction-error based methods can naturally benefit from such nature of the game. In contrast, SSE-DB encourages exploration based on the information gain from learning the dynamics-relevant representation, which may not capture such critical events in Tennis.

\subsection{Robustness in the Presence of Noises}

\paragraph{Observation Noises.} To analyze the robustness of SSE-DB to observation noises, an important evaluation metric is the performance of SSE-DB in the presence of dynamics-irrelevant information. A particularly challenging distractor is the white-noise \citep{largescale-2019,CB-2019}, which incorporates random task-irrelevant patterns to the observations. In such a scenario, a frequently visited state by injecting an unseen noise pattern may be mistakenly assigned with a high intrinsic reward by curiosity or pseudo-count based methods.

\begin{wrapfigure}{r}{0cm}
\centering
\includegraphics[width=0.45\textwidth]{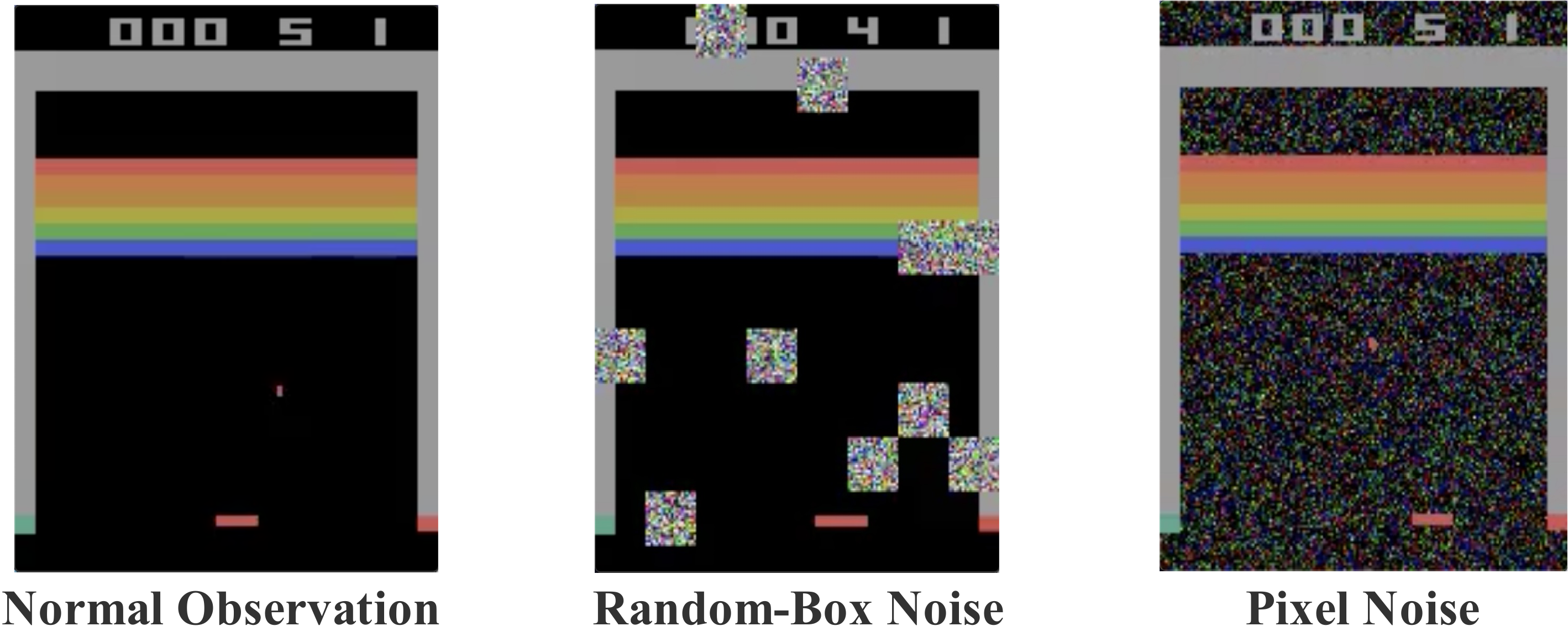}
\caption{Observation of Breakout with no distractor (left), random-box noise distractor (middle), and pixel noise distractor (right).}
\label{fig:random-state}
\end{wrapfigure}

We use two types of distractors for the observations of Atari games, namely, (1) the random-box noise distractor, which places boxes filled with random Gaussian noise over the raw pixels, and (2) the pixel-level noise distractor, which adds pixel-wise Gaussian noise to the observations. Fig.~\ref{fig:random-state} shows examples of the two types of distractors. In the sequel, we discuss results for the random-box noise distractor on selected Atari games, which we find sufficiently representative, and defer the complete report to Appendix~\ref{app:experiment-box} and \ref{app:experiment-pixel}.

Fig.~\ref{fig:result-random-box} shows the performance of the compared methods on Alien, Breakout and TimePilot with and without noises. We observe that SSE-DB outperforms ICM on Alien and TimePilot with random-box noises. Nevertheless, in Breakout, we observe that both the methods fail to learn informative policies. A possible explanation is that, in Breakout, the ball is easily masked by the box-shaped noise (i.e., middle of Fig.~\ref{fig:random-state}). The random-box noise therefore buries critical transition information of the ball, which hinders all the baselines to extract dynamics-relevant information and leads to failures on Breakout with random-box noise as shown in Fig.~\ref{fig:result-random-box}.

\begin{figure}[!t]
\centering
\subfigure[Atari with random-box noise]{\includegraphics[width=2.6in]{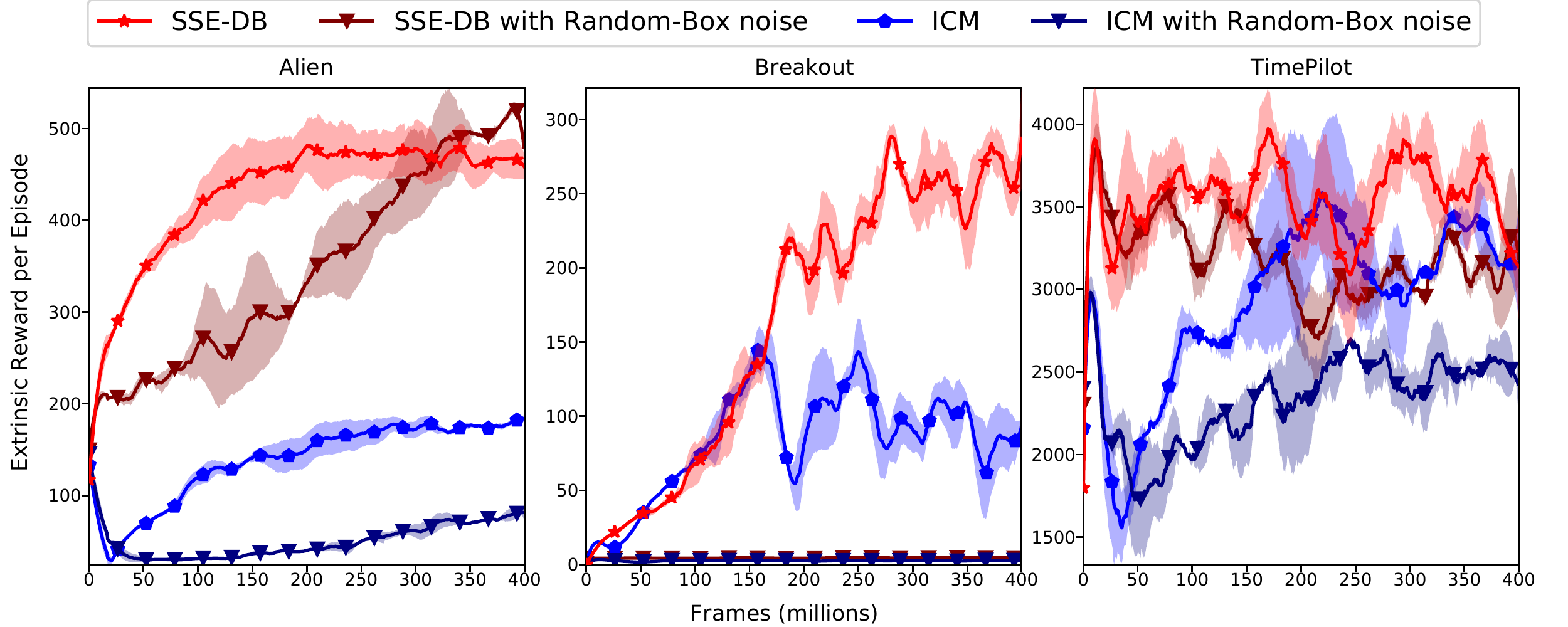}\label{fig:result-random-box}}
\subfigure[Atari with sticky actions]{\includegraphics[width=2.6in]{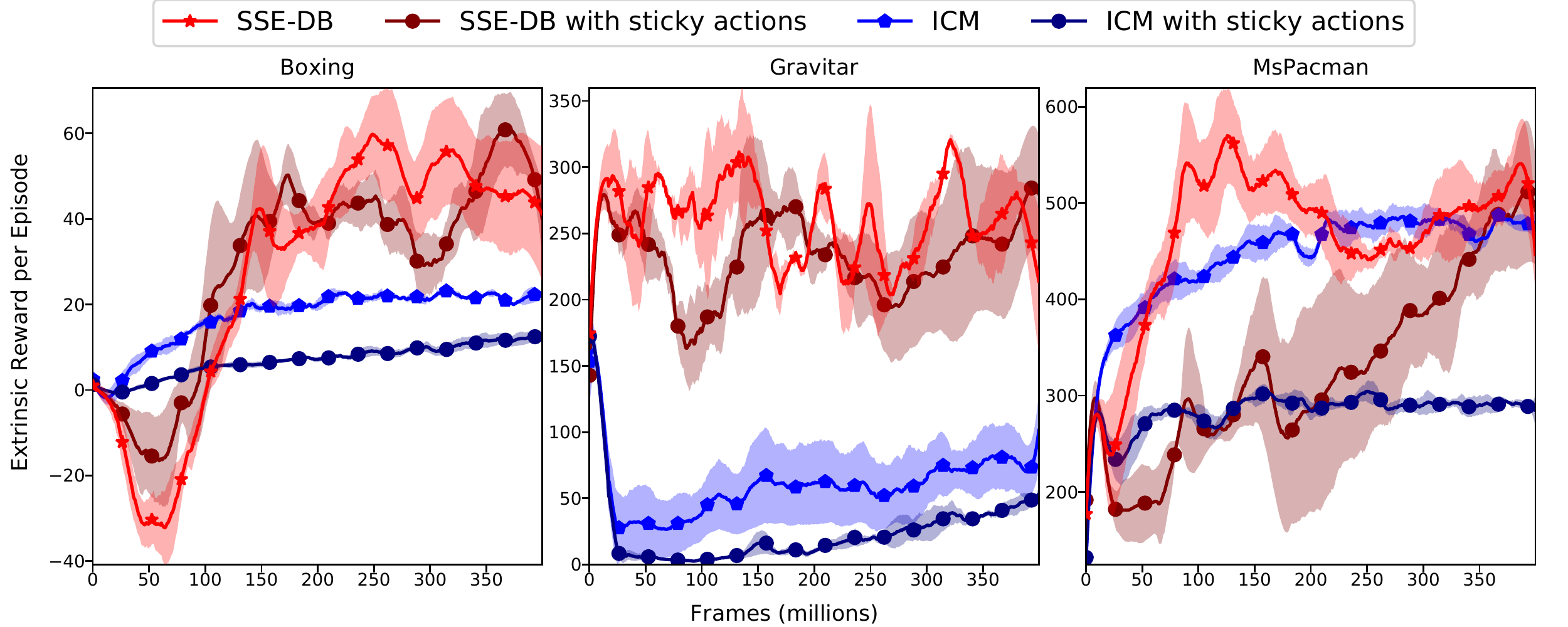}\label{fig:result-sticky}}
\caption{A comparison on selected Atari games with and without noises. (a) Random-box noise. In Alien and TimePilot, SSE-DB is barely affected by random-box noise. Both methods fail in Breakout. (b) Sticky actions. Illustration shows SSE-DB is barely affected by the sticky actions.}
\end{figure}

\paragraph{Action Noise.} In addition to observation noises, noises in actions also raise challenges for learning the transition dynamics. To further study the robustness of SSE-DB, we conduct experiments on Atari games with sticky actions \citep{atari-2018,disagree-2019}. At each time step, the agent may execute the previous action instead of the output action of current policy with a probability of $0.25$. We illustrate results on three selected Atari games, i.e., Boxing, Gravitar and MsPacman, in Fig.~\ref{fig:result-sticky} and defer the complete results to Appendix~\ref{app:experiment-sticky}. Our experiments show that SSE-DB is robust to action noises, whereas ICM suffers significant performance drop from the action noise.

\subsection{Visualization and Ablation Study}

\begin{figure}[t]
\centering
\includegraphics[width=3.6in]{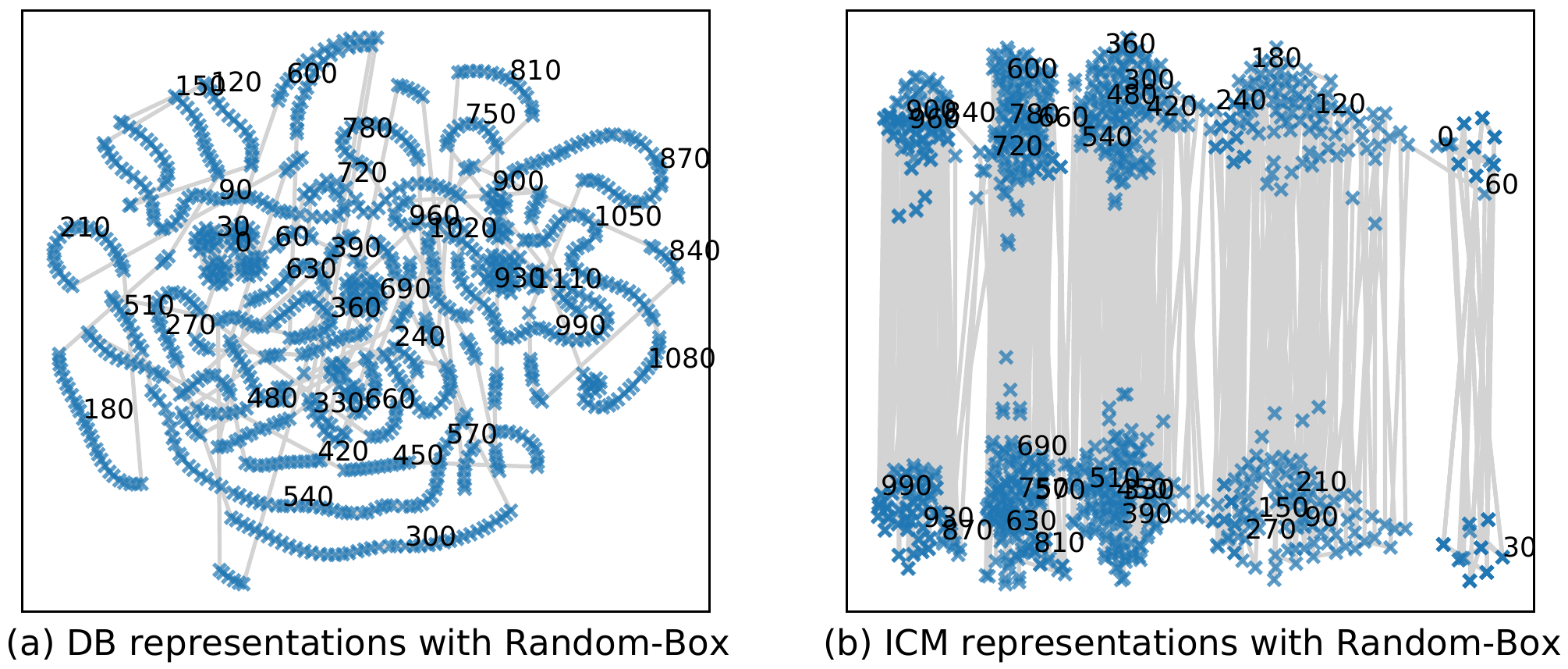}
\caption{Visualization of latent representations of MsPacman with random-box noise. The figures plot the representation of the same trajectory by DB (left) and ICM (right). Consecutive states are connected by the shaded lines. Numbers on the representations are the corresponding number of time steps. Representation learned by DB tends to align consecutive states on the same curve.}
\label{fig:tsne}
\end{figure}

\paragraph{Visualization of the Learned Representations.} To understand the latent representation learned by the DB model, we visualize the learned $Z$ with t-SNE \cite{tsne-2008} plots, which projects a 128d $z$-vector to a 2d one through dimensionality reduction. We compare the representations learned by SSE-DB and ICM with random-box noise. We illustrate the learned representations of MsPacman in Fig.~\ref{fig:tsne}. According to the visualization, representations learned by DB tends to align temporally-consecutive movements on the same curve. Moreover, each segment of a curve corresponds to a semantic component of the trajectory, such as eating pellets aligned together or avoiding ghosts. The segments of a curve end up with critical states, including death and reborn of the agent and ghosts. The visualization indicates that DB well captures the dynamics-relevant information in the learned representations. In contrast, such temporally-consecutive patterns are missing in the learned representation of ICM.

\paragraph{Visualization of the DB bonus.} We provide visualization of the DB-bonus in Appendix~\ref{app:vis-bonus}. The results show the DB-bonus effectively encourages the agent to explore the informative transitions.

\paragraph{Ablation Study.} The training of DB loss consists of multiple components, including $I_{\rm pred}$, $I_{\rm nce}$, $I_{\rm upper}$, and the momentum observation encoder. To analyze the importance of the components, we conduct an ablation study by removing each of them respectively and evaluate the DB model correspondingly. The ablation study suggests that the all the components are crucial for learning effective dynamics-relevant representations. In addition, we observe that $I_{\rm upper}$ is particularly important in the environment with dynamics-irrelevant noise. Please refer to Appendix~\ref{app:experiment-abla} for details.

\section{Conclusion}

In this paper, we introduce Dynamic Bottleneck model that learns dynamics-relevant representations based on the IB principle. 
Based on the DB model, we further propose DB-bonus based on the DB model for efficient exploration. We establish theoretical connections between the proposed DB-bonus and provably efficient bonuses. Our experiments show that SSE-DB outperforms several strong baselines in stochastic environments for self-supervised exploration. Moreover, we observe that DB learns well-structured representations and the DB-bonus characterizes informative transitions for exploration.
For our future work, we wish to combine the DB representation to effective exploration methods including BeBold \citep{bebold-2020} and NGU \citep{ngu-2020} to enhance their robustness in stochastic environments.

\section*{Acknowledgements}

The authors thank Tencent Robotics X and Vector Institute for the computation resources supported. Part of the work was done during internship at Tencent Robotics X lab. The authors also thank the anonymous reviewers, whose invaluable suggestions have helped us to improve the paper.

\bibliography{db-neurips}

\appendix

\clearpage

\title{Dynamic Bottleneck for Robust Self-Supervised Exploration (Appendix)}

\section{Proof of Equation \ref{eq:nce}}\label{app:nce}
\begin{proof}
We introduce the latent variable $C$ to indicate whether the next-state encoding $s_{t+1}$ the representation $z_t$ are drawn from the joint density ($C=1$) or from the product of marginals ($C=0$). For the positive sample with $C=1$, we have 
\[p(z_t,s_{t+1}|C=1)=p(z_t,s_{t+1}),\]
which is the joint density. For the negative example, we have 
\[p(z_t,s_{t+1}|C=0)=p(z_t)p(s_{t+1}),\]
which is the product of marginals. In the sequel, we use contrastive objective containing one positive pair and $N$ negative pairs. Correspondingly, the priors of latent $C$ takes the form of 
\[p(C=1)=1/(N+1),\:\:p(C=0)=N/(N+1).\]
The following computation is adopted from \cite{cl-2018}. By Bayesian rule, the posterior of $C=1$ takes the form of
\begin{equation}\label{eq::1111}
\begin{split}
&\log p(C=1|z_t,s_{t+1})\\
&=\log\frac{p(C=1) p(z_t,s_{t+1}|C=1)}{p(C=0) p(z_t,s_{t+1}|C=0)+p(C=1)p(z_t,s_{t+1}|C=1)}\\
&=\log\frac{p(C=1) p(z_t,s_{t+1})}{p(C=0)p(z_t)p(s_{t+1})+p(C=1)p(z_t,s_{t+1})}\\
&=\log\frac{p(z_t,s_{t+1})}{Np(z_t)p(s_{t+1})+p(z_t,s_{t+1})}\\
&=-\log (1+N\frac{p(z_t)p(s_{t+1})}{p(z_t,s_{t+1})})\\
&\leq-\log N+\log \frac{p(z_t,s_{t+1})}{p(z_t)p(s_{t+1})}.
\end{split}
\end{equation}
By taking expectation with respect to $p(z_t,s_{t+1}|C=1)$ on both sides of \eqref{eq::1111} and rearranging, we obtain
\begin{equation}
I(Z_t,S_{t+1})\geq \log N+\mathbb{E}_{p(z_t,s_{t+1})}[\log p(C=1|z_t,s_{t+1})],
\end{equation}
where $I(Z_t;S_{t+1})=\mathbb{E}_{p(z_t,s_{t+1})}[\log \frac{p(z_t,s_{t+1})}{p(z_t)p(s_{t+1})}]$ is the mutual information between the distributions of the dynamics-relevant representation $Z_t$ and next-observation encoding $S_{t+1}$. In the sequel, we follow \cite{cl-2018} and fit a score function $h(z_t,s_{t+1})$ to maximize the log-likelihood of positive samples by solving a binary classification problem,
\begin{equation}
\mathcal{L}_{\rm nce}(h)=\mathbb{E}_{p(z_t,s_{t+1})}\mathbb{E}_{S^-}\bigg[\log \frac{\exp (h(z_t,s_{t+1}))}{\sum_{s_j\in S^-\cup{s_{t+1}}}\exp (h(z_t,s_j))}\bigg].
\end{equation}
where the denominator involves both the positive and negative pairs. If $h$ is sufficiently expressive, the optimal solution of the binary classifier is $h^*(z_t,s_{t+1})=p(C=1|z_t,s_{t+1})$.
Thus, we have 
\begin{equation}\label{eq::2222}
\begin{split}
I(Z_t,S_{t+1})&\geq \log N+\mathbb{E}_{p(z_t,s_{t+1})}[\log p(C=1|z_t,s_{t+1})]\\
&= \log N+\mathbb{E}_{p(z_t,s_{t+1})} [\log h^*(z_t,s_{t+1})]\\
&\geq \log N+\mathbb{E}_{p(z_t,s_{t+1})} [\log h^*(z_t,s_{t+1})-\log \sum\nolimits_{s_j\in S^-\cup{s_{t+1}}}\exp (h^*(z_t,s_j))]\\
&= \log N+\mathcal{L}_{\rm nce}(h^*)=\log N+\max_h \mathcal{L}_{\rm nce}(h)\\
& \geq \log N+\mathcal{L}_{\rm nce}(h).
\end{split}
\end{equation}
The third line holds since $h^*(z_t,s_{t+1})\in [0,1]$ and, hence, the added term is strictly negative. Since $N$ is a constant decided by the training batch in DB, it suffices to maximize $\mathcal{L}_{\rm nce}(h)$ in \eqref{eq::2222} to maximize the mutual information $I(Z_t,S_{t+1})$.
\end{proof}

\newpage
\section{Pseudocode for DB training}\label{app:db-train}

The network of DB in Fig.~\ref{fig:db-network} contains four modules, we introduce the network architecture as follows.
\begin{itemize}[leftmargin=10pt]
\item \emph{Observation Encoder $f^S_o(\cdot;\theta_o)$ and $f^S_m(\cdot;\theta_m)$}. The observation encoder contains a main network and a momentum network, which are used to extract features of $o_t$ and $o_{t+1}$, respectively. Each network contains three convolution layers as Conv(filter=32, kernel-size=8, strides=4) $\rightarrow$ Conv(filter=64, kernel-size=4, strides=2)$\rightarrow$ Conv(filter=64, kernel-size=3, strides=1). We adopt Leaky ReLU as teh activation function. The architecture is similar to DQN without FCNs \citep{DQN-2015}. The input to observation encoder is image and the output is a vector $s\in \mathbb{R}^{512}$.
\item \emph{Representation Posterior $g^Z(\cdot;\phi)$}. The posterior $g^Z(s_t,a_t;\phi)$ represents the dynamics-relevant information extracted from state and action. We concatenate $s_t\in \mathbb{R}^{512}$ and one-hot vector $a_t\in\mathbb{R}^{|\mathcal{A}|}$ as the input. The processing flow is FCN(units=256) $\rightarrow$ ResNet(units=(256,256)) $\rightarrow$ ResNet(units=(256,256)) $\rightarrow$ FCN(units=256), where ResNet is a residual network with two layers. The output of $g^Z$ is a diagonal Gaussian with mean $\mu_z\in \mathbb{R}^{128}$ and variance $\sigma_z\in \mathbb{R}^{128}$. We use soft-plus activation to make the variance positive. The representation $z\in \mathbb{R}^{128}\sim \mathcal{N}(\mu_z,\sigma_z)$.
\item \emph{Prediction Head $q(\cdot;\psi)$}. The prediction head $q(z;\psi)$ contains three ResNets and two FCNs. We process the input $z\in \mathbb{R}^{128}$ by FCN(units=256) $\rightarrow$ ResNet(units=(256,256)) $\rightarrow$ ResNet(units=(256,256)) $\rightarrow$ FCN(units=512) $\rightarrow$ ResNet(units=(512,512)), where ResNet is a residual network with two layers. The output is a mean vector, which has the same dimensions as $s_{t+1}$, and an additional variance estimation. The output of $q_{\psi}$ is a diagonal Gaussian that has the same variance in each dimension.
\item \emph{Projection head $f^P_o(\cdot;\varphi_o)$ and $f^P_m(\cdot;\varphi_m)$}. The projection heads map $\bar{q}(z_t)$ and $s_{t+1}$ to low-dimensional space for contrastive estimation. The projection head contains two FCNs with 256 and 128 units, respectively. We follow \citet{simclr-2020} and adopt a normalization layer at each layer.
\end{itemize}

\begin{algorithm*}[h]
\caption{Pseudocode for DB training, PyTorch-like}
\label{alg:code}
\definecolor{codeblue}{rgb}{0.25,0.5,0.5}
\definecolor{codekw}{rgb}{0.85, 0.18, 0.50}
\lstset{
  backgroundcolor=\color{white},
  basicstyle=\fontsize{9.0pt}{9.0pt}\ttfamily\selectfont,
  columns=fullflexible,
  breaklines=true,
  captionpos=b,
  commentstyle=\fontsize{9.0pt}{9.0pt}\color{codeblue},
  keywordstyle=\fontsize{9.0pt}{9.0pt}\color{codekw},
}
\begin{lstlisting}[language=python,mathescape=true]
# $f^S_o$, $f^S_m$: observation encoders, $g^Z$: representation posterior, $q$: prediction head
# $f^P_o$, $f^P_m$: projection heads, $\mathcal{W}$: weight matrix in contrastive loss

for (o, a, o_next) in loader: # Load 16 episodes from 16 actors. N=128*16-1.
    s, s_next = $f^S_o$(o), $f^S_m$(o_next).detach()   # observation encoder 
    z_dis = $g^Z$(s,a)              # Gaussian distribution of Z
    I_upper = KL(z_dis, $N(0,I)$)   # KL-divergence to compress representation (a upper bound)

    z = r_sample(z_dis)   # reparameterization
    s_pred_dis = q(z)   # prediction head, the output is a Gaussian distribution
    I_pred = log(s| s_pred_dis)   # predictive objective (lower bound)

    s_pred = mean(s_pred_dis)       # take the mean value of prediction      
    s_pred_proj = $f^P_o$(s_pred)   # projection head
    s_next_proj = $f^P_m$(s_next).detach()  # momentum projection head
    
    logits = matmul(s_pred_proj, matmul($\mathcal{W}$, s_next_proj.T))  # (N-1) x (N-1)
    logits = logits - max(logits, axis=1) # subtract max from logits for stability
    labels = arange(logits.shape[0]) 
    I_nce = -CrossEntropyLoss(logits, labels)   # contrastive estimation (lower bound)
    
    L = $\alpha_1\cdot$I_upper - $\alpha_2\cdot$I_pred - $\alpha_3\cdot$I_nce  # total loss function
    L.backward()  # back-propagate
    update($f^S_o$, $g^Z$, $q$, $f^P_o$, $\mathcal{W}$)   # Adam update the parameters
    $\theta_m$ = $\tau\cdot$ $\theta_m$+(1-$\tau$) $\cdot \theta_o$
    $\theta_m$ = $\tau\cdot$ $\varphi_m$+(1-$\tau$) $\cdot \varphi_m$
\end{lstlisting}
\end{algorithm*}

We use 128 actors in experiments, and each episode contains 128 steps. Since the batch size 128*128 is too large for GPU memory (RTX-2080Ti), we follow the implementation of ICM and Disagreement by using experiences from 16 actors for each training step. We iterate 8 times to sample all experiences from 128 actors. As a result, the corresponding negative sample size $|S^{-}|$ is $16*128-1$ for contrastive estimation.
\newpage

\section{Proof of the DB-bonus}\label{app:proof-bonus}

\subsection{Background: LSVI-UCB}\label{app:optimistic-lsvi}

The algorithmic description of LSVI-UCB \citep{jin-2019} is given in Algorithm~\ref{alg:lsvi}. The feature map of the state-action pair is denoted as $\eta:\mathcal{S}\times\mathcal{A}\rightarrow\mathbb{R}^d$. The transition kernel and reward function are assumed to be linear in $\eta$. Under such a setting, it is known that for any policy $\pi$, the corresponding action-value function $Q_t(s,a)=\chi_t^{\top}\eta(s,a)$ is linear in $\eta$ \citep{jin-2019}. Each iteration of LSVI-UCB consists of two parts. First, in line 3-6, the agent executes the policy according to $Q_t$ for an episode. Second, in line 7-11, the parameter $\chi_t$ of $Q$-function is updated in closed-form by following the regularized least-squares 
\begin{equation*}
\chi_t\leftarrow\arg\min_{\chi\in\mathbb{R}^d}\sum\nolimits_{i=0}^{m}\bigl[r_t(s_t^i,a_t^i)+\max_{a\in\mathcal{A}}Q_{t+1}(s_{t+1}^{i},a)-\chi^{\top}\eta(s_t^{i},a_t^{i})\bigr]^2+\lambda\|\chi\|^2,
\end{equation*}
where $m$ is the number of episodes, and $i$ is the index of episodes. The least-squares problem has the following closed form solution ,
$$
\chi_t=\Lambda_t^{-1}\sum_{\tau=0}^{m}\eta(x_t^{i},a_t^{i})\bigl[r_t(x_t^{i},a_t^{i})+\max_a Q_{t+1}(x_{t+1}^{i},a)\bigr],
$$ 
where $\Lambda_t$ is the Gram matrix. The action-value function is estimated by $Q_t(s,a)\approx \chi_t^{\top}\eta(s,a)$. 

LSVI-UCB uses UCB-bonus (line 10) to construct the confidence bound of $Q$-function as $r^{\rm ucb}=\beta\big[\eta(s,a)^\top\Lambda_t^{-1}\eta(s,a)\big]^{\nicefrac{1}{2}}$ \citep{bandit-2011}, which measures the epistemic uncertainty of the corresponding state-action pairs. Theoretical analysis shows that LSVI-UCB achieves a near-optimal worst-case regret of $\tilde{\mathcal{O}}(\sqrt{d^3 T^3 L^3})$ with proper selections of $\beta$ and $\lambda$, where $L$ is the total number of steps. We refer to \citet{jin-2019} for the detailed analysis. 

\begin{algorithm}[h]
\caption{LSVI-UCB for linear MDP}
\label{alg:lsvi}
\begin{algorithmic}[1]
\STATE {{\bf Initialize:} $\Lambda_t\leftarrow\lambda\cdot\mathbf{I}$ and $w_h\leftarrow 0$}
\FOR {episode $m=0$ {\bfseries to} $M-1$}
\STATE {Receive the initial state $s_0$}\label{alg:lsvi-interact}
\FOR {step $t=0$ {\bfseries to} $T-1$}
\STATE {Take action $a_t=\arg\max_{a\in\mathcal{A}}Q_t(s_t,a)$ and observe $s_{t+1}$.}
\ENDFOR\label{alg:lsvi-end-interact}
\FOR {step $t=T-1$ {\bfseries to} $0$}\label{alg:lsvi-train}
\STATE {$\Lambda_t\leftarrow\sum_{i=0}^{m}\eta(x_t^{i},a_t^{i})\eta(x_t^{i},a_t^{i})^\top+\lambda \cdot \mathrm{\mathbf{I}}$}
\STATE {$\chi_t\leftarrow\Lambda_t^{-1}\sum_{i=0}^{m}\eta(x_t^{i},a_t^{i})[r_t(x_t^{i},a_t^{i})+\max_a Q_{t+1}(x_{t+1}^{i},a)]$}\label{alg:lsvi-solve}
\STATE {$Q_t(\cdot,\cdot)=\min\{\chi_t^\top\eta(\cdot,\cdot)+\alpha[\eta(\cdot,\cdot)^\top\Lambda_t^{-1}\eta(\cdot,\cdot)]^{\nicefrac{1}{2}},T\}$}\label{alg:lsvi-bonus}
\ENDFOR\label{alg:lsvi-end-train}
\ENDFOR
\end{algorithmic}
\end{algorithm}

\subsection{Proof of connection to LSVI-UCB}\label{app:proof-ucb}

In linear function approximation, we set the representation of DB to be linear in state-action encoding, namely, $z_t = W_t\eta(s_t,a_t)\in \mathbb{R}^{c}$, where $W_t\in \mathbb{R}^{c\times d}$ and $\eta(s_t,a_t)\in\mathbb{R}^{d}$. To capture the prediction error in DB, we conduct regression to recover the next state $s_{t+1}$ and consider the following regularized least-square problem,
\begin{equation}\label{eq::regression_problem}
w_t\leftarrow\arg\min_{W}\sum_{i=0}^{m}\bigl\|s_{t+1}^i-W_t\eta(s_t^i,a_t^i)\bigr\|^2_F + \lambda \|W\|^2_F,
\end{equation}
where $\|\cdot\|_F$ denotes the Frobenius norm. In the sequel, we consider a Bayesian linear regression perspective of (\ref{eq::regression_problem}) that captures the intuition behind the DB-bonus. Our objective is to approximate the next-state prediction $s_{t+1}$ via fitting the parameter $W$, such that
\[
W \eta(s_t, a_t) \approx s_{t+1},
\]
where $s_{t+1}$ is given. We assume that we are given a Gaussian prior of the initial parameter $W \sim \mathcal N(\mathbf{0}, \mathrm{\mathbf{I}}/\lambda)$. With a slight abuse of notation, we denote by $W_t$ the Bayesian posterior of the parameter $W$ given the set of independent observations $\mathcal{D}_m = \{(s^i_t, a^i_t, s^i_{t+1})\}_{i \in [0,m]}$. 
We further define the following noise with respect to the least-square problem in (\ref{eq::regression_problem}),
\begin{equation}\label{eq:noise}
\epsilon = s_{t+1} - W_t \eta(s_t, a_t)\in\mathbb{R}^{c},
\end{equation}
where $(s_t, a_t, s_{t+1})$ follows the distribution of trajectory. The following theorem justifies the DB-bonus under the Bayesian linear regression perspective.
\begin{thm}[Formal Version of Theorem \ref{thm:ucb-informal}]
\label{thm:ucb-formal}
We assume that $\epsilon$ follows the standard multivariate Gaussian distribution $\mathcal N(0, \mathbf{I})$ given the  state-action pair $(s_t, a_t)$ and the parameter $W$. Assuming $W$ follows the Gaussian prior $\mathcal{N}(0, \mathrm{\mathbf{I}}/\lambda)$. We define
\begin{equation}\label{eq::def_lambda}
\Lambda_t=\sum_{i=0}^{m}\eta(x_t^{i},a_t^{i})\eta(x_t^{i},a_t^{i})^\top+\lambda \cdot \mathrm{\mathbf{I}}.
\end{equation}
It then holds for the posterior of $W_t$ given the set of independent observations $\mathcal{D}_m = \{(s^i_t, a^i_t, s^i_{t+1})\}_{i \in [0,m]}$ that
\[\sqrt{\frac{c}{4}}\:\big[\eta(t)^{\top}\Lambda_t^{-1}\eta(t)\big]^{\nicefrac{1}{2}} \leq I(W_t;[s_t,a_t, S_{t+1}]|\mathcal{D}_m)^{\nicefrac{1}{2}} \leq \sqrt{\frac{c}{2}}\:\big[\eta(t)^{\top}\Lambda_t^{-1}\eta(t)\big]^{\nicefrac{1}{2}}.\]

\end{thm}

\begin{proof}
The proof follows the standard analysis of Bayesian linear regression. See, e.g., \citet{west1984outlier} for the details. We introduce the following notations for $W_t\in\mathbb{R}^{c\times d}$, $\eta(s_t,a_t)\in \mathbb{R}^{d}$, and $W_t\eta(s_t,a_t)\in\mathbb{R}^c$,
\begin{equation}
W_t=
\begin{bmatrix}
w_{11} & \cdots & w_{1d}\\
\vdots &        & \vdots\\
w_{c1} & \cdots & w_{cd}
\end{bmatrix}\in\mathbb{R}^{c\times d},\quad
\eta(s_t,a_t)=
\begin{bmatrix}
\eta_{1}\\
\eta_{2}\\
\vdots\\
\eta_{d}
\end{bmatrix}\in\mathbb{R}^d,
\quad
W_t\eta(s_t,a_t)=
\begin{bmatrix}
\begin{smallmatrix}
\sum_{k=1}^d w_{1k}\eta_k\\
\sum_{k=1}^d w_{2k}\eta_k\\
\vdots\\
\sum_{k=1}^d w_{ck}\eta_k\\
\end{smallmatrix}
\end{bmatrix}\in\mathbb{R}^c.
\end{equation}

For the analysis, we vectorize matrix $W_t$ and define a new matrix $\tilde{\eta}$. Meanwhile, we define $\tilde{\eta}$ by repeating $\eta(s_t,a_t)$ for $c$ times in the diagonal. Specifically, we define ${\rm vec}(W_t) \in \mathbb{R}^{cd}$ and $\tilde{\eta}(s_t,a_t)\in \mathbb{R}^{cd\times c}$ as follows,
\begin{equation}\label{eq::vecw}
{\rm vec}(W_t)=
\begin{bmatrix}
\begin{smallmatrix}
w_{11}\\
\vdots\\
w_{1d}\\
w_{21}\\
\vdots\\
w_{2d}\\
\vdots\\
\vdots\\
w_{c1}\\
\vdots\\
w_{cd}
\end{smallmatrix}
\end{bmatrix}\in\mathbb{R}^{cd},\quad
\tilde{\eta}(s_t,a_t)=
\begin{bmatrix}
\begin{smallmatrix}
\eta(s_t,a_t) & 0             & \cdots & 0            \\
0             & \eta(s_t,a_t) & \cdots & 0            \\
\vdots        & \vdots        & \ddots & \vdots       \\
0             & 0             & \cdots & \eta(s_t,a_t)\\
\end{smallmatrix}
\end{bmatrix}=
\begin{bmatrix}
\begin{smallmatrix}
\eta_{1} & 0        & \cdots & 0       \\
\vdots   &          &        &         \\
\eta_{d} & 0        & \cdots & 0       \\
         & \eta_{1} & \cdots & 0       \\ 
\vdots   & \vdots   &        & \vdots  \\
         & \eta_{d} & \cdots & 0       \\
\vdots   & \vdots   &        & \vdots  \\
\vdots   & \vdots   &        & \vdots  \\
0        & 0        & \cdots & \eta_{1}\\
\vdots   & \vdots   &        & \vdots  \\
0        & 0        & \cdots & \eta_{d}
\end{smallmatrix}
\end{bmatrix}\in \mathbb{R}^{cd\times c},
\end{equation}
Then we have ${\rm vec}(W_t)^{\top}\tilde{\eta}(s_t,a_t)=W_t\eta(s_t,a_t)$ according to block matrix multiplication and \eqref{eq::vecw}. Formally,
\begin{equation}\label{eq::vecw-w}
{\rm vec}(W_t)^{\top}\tilde{\eta}(s_t,a_t)=
\begin{bmatrix}
\sum_{k=1}^d w_{1k}\eta_k \\
\sum_{k=1}^d w_{2k}\eta_k \\
\vdots \\
\sum_{k=1}^d w_{ck}\eta_k \\
\end{bmatrix}^\top=W_t\eta(s_t,a_t)\in\mathbb{R}^c.
\end{equation}

Our objective is to compute the posterior density $W_t = W | \mathcal{D}_m$, where $\mathcal{D}_m = \{(s^i_t, a^i_t, s^i_{t+1})\}_{i \in [0,m]}$ is the set of observations. The target of the linear regression is $s_{t+1}$. By the assumption that $\epsilon$ follows the standard Gaussian distribution, we obtain that 
\begin{equation}\label{eq:pf_density_y}
s_{t+1} | (s_t, a_t), W_t \sim \mathcal{N}\bigl(W_t \eta(s_t, a_t), \mathbf{I}\bigr).
\end{equation}
Because we have ${\rm vec}(W_t)^{\top}\tilde{\eta}(s_t,a_t)=W_t\eta(s_t,a_t)$ according to \eqref{eq::vecw-w}, we have 
\begin{equation}\label{eq:pf_density_y2}
s_{t+1} | (s_t, a_t), W_t \sim \mathcal{N}\bigl({\rm vec}(W_t)^{\top} \tilde{\eta}(s_t, a_t), \mathbf{I}\bigr).
\end{equation}

Recall that we have the prior distribution $W \sim \mathcal{N}(0, \mathrm{\mathbf{I}}/\lambda)$, then the prior of ${\rm vec}(W) \sim \mathcal{N}(0, \mathrm{\mathbf{I}}/\lambda)$. It holds from Bayes rule that 
\begin{equation}\label{eq:pf_bayes_rule}
\log p({\rm vec}(W_t) | \mathcal{D}_m) = \log p({\rm vec}(W_t)) + \log p(\mathcal{D}_m | {\rm vec}(W_t)) + Const.
\end{equation}
Plugging \eqref{eq:pf_density_y2} and the probability density function of Gaussian distribution into \eqref{eq:pf_bayes_rule} yields
\begin{equation}\label{eq:pf_density_posterior}
\begin{split}
\log p({\rm vec}(W_t) | \mathcal{D}_m) &= - \|{\rm vec}(W_t)\|^2/2 -\sum^m_{i=1} \| {\rm vec}(W_t)\tilde{\eta}(s^i_t, a^i_t) - s^i_{t+1}\|^2/2 + Const\\
&=-({\rm vec}(W_t)-\tilde{\mu}_t)^\top \tilde{\Lambda}^{-1}_t({\rm vec}(W_t)-\tilde{\mu}_t)/2 + Const,
\end{split}
\end{equation}
where we define 
\[
\tilde{\mu}_t = \tilde{\Lambda}^{-1}_t \sum^m_{i=0}\tilde{\eta}(s^i_t, a^i_t) s^i_{t+1}\in\mathbb{R}^{cd}, \qquad \tilde{\Lambda}_t=\sum_{i=0}^{m}\tilde{\eta}(s_t^i,a_t^i)\tilde{\eta}(x_t^i,a_t^i)^\top+\lambda \cdot \mathrm{\mathbf{I}}\in\mathbb{R}^{cd\times cd}.
\]
Thus, by \eqref{eq:pf_density_posterior}, we obtain that ${\rm vec}(W_t) = W| \mathcal{D}_m \sim \mathcal N(\tilde{\mu}_t, \tilde{\Lambda}^{-1}_t)$. We have the covariance matrix of ${\rm vec}(W_t)$ is
\begin{equation}\label{eq::var-lambda}
{\rm Var}({\rm vec}(W_t))=\tilde{\Lambda}_t^{-1}
\end{equation}

The $\tilde{\Lambda}_t$ term accumulates previous $\tilde{\eta}(s^i_t,a^i_t)\tilde{\eta}(s^i_t,a^i_t)^{\top}$. Since we have
\begin{equation}
\begin{split}
\tilde{\eta}(s^i_t,a^i_t)\tilde{\eta}(s^i_t,a^i_t)^{\top}&=
\begin{bmatrix}
\begin{smallmatrix}
\eta(t) & 0 & \cdots & 0     \\
 0     & \eta(t) & \cdots & 0     \\
\vdots & \vdots & \ddots & \vdots\\
0      & 0 & \cdots & \eta(t)\\
\end{smallmatrix}
\end{bmatrix}
\begin{bmatrix}
\begin{smallmatrix}
\eta(t)^{\top} & 0      & \cdots & 0     \\
 0     & \eta(t)^{\top} & \cdots & 0         \\
\vdots & \vdots         & \ddots & \vdots\\
0      & 0              & \cdots & \eta(t)^{\top}\\
\end{smallmatrix}
\end{bmatrix}\\
&=
\begin{bmatrix}
\begin{smallmatrix}
\eta(t)\eta(t)^{\top} & 0 & \cdots & 0     \\
 0     & \eta(t)\eta(t)^{\top} & \cdots & 0     \\
\vdots & \vdots & \ddots & \vdots\\
0      & 0 & \cdots & \eta(t)\eta(t)^{\top}\\
\end{smallmatrix}
\end{bmatrix}\in\mathbb{R}^{cd\times cd},
\end{split}
\end{equation}
where $\eta(t)=\eta(s^i_t,a^i_t)\in\mathbb{R}^d$, and $\eta(t)\eta(t)^{\top}\in\mathbb{R}^{d\times d}$ repeats for $c$ times, we further expand the matrix $\tilde{\Lambda}_t$ as,
\begin{equation}\label{eq::tildelambda}
\begin{split}
\tilde{\Lambda}_t=\sum_{i=0}^{m}\tilde{\eta}(s_t^i,a_t^i)\tilde{\eta}(x_t^i,a_t^i)^\top+\lambda \mathrm{\mathbf{I}}	&=
\begin{bmatrix}
\begin{smallmatrix}
\sum\eta(t)\eta(t)^{\top}+\lambda I & 0 & \cdots & 0     \\
 0     & \sum\eta(t)\eta(t)^{\top}+\lambda I & \cdots & 0     \\
\vdots & \vdots & \ddots & \vdots\\
0      & 0 & \cdots & \sum\eta(t)\eta(t)^{\top}+\lambda I\\
\end{smallmatrix}
\end{bmatrix}\\
&=
\begin{bmatrix}
\Lambda_t & 0 & \cdots & 0     \\
 0        & \Lambda_t & \cdots & 0     \\
\vdots    & \vdots & \ddots & \vdots\\
0         & 0 & \cdots & \Lambda_t \\
\end{bmatrix},
\end{split}
\end{equation}
where we follow the definition of $\Lambda_t$ in \eqref{eq::def_lambda}. Further, the mutual information
\begin{equation}\label{vecw-info}
\begin{split}
I({\rm vec}(W_t)&;[s_t,a_t, S_{t+1}]|\mathcal{D}_m)=\mathcal{H}({\rm vec}(W_t)|\mathcal{D}_m)-\mathcal{H}({\rm vec}(W_t)|(s_t,a_t, S_{t+1})\cup\mathcal{D}_m)\\
&=\frac{1}{2}\log\det\big({\rm Var}({\rm vec}(W_t)|\mathcal{D}_m)\big)-\frac{1}{2}\log \det\big({\rm Var}({\rm vec}(W_t)|(s_t,a_t, S_{t+1})\cup\mathcal{D}_m)\big).
\end{split}
\end{equation}
Plugging \eqref{eq::var-lambda} into \eqref{vecw-info}, we have
\begin{equation}\label{eq::info-w}
\begin{split}
I({\rm vec}(W_t);[s_t,a_t, S_{t+1}]|\mathcal{D}_m)&=\frac{1}{2}\log\det\big(\tilde{\Lambda}_t^{-1}\big)-\frac{1}{2}\log \det\big((\tilde{\Lambda}_{t}^{\dagger})^{-1}\big)\\
&=\frac{1}{2}\log\det\big(\tilde{\Lambda}_{t}^\dagger\big)-\frac{1}{2}\log \det\big(\tilde{\Lambda}_t\big)\\
&=\frac{1}{2}\log\det\big(\tilde{\Lambda}_t+\tilde{\eta}(s_t,a_t)\tilde{\eta}(s_t,a_t)^{\top}\big)-\frac{1}{2}\log \det\big(\tilde{\Lambda}_t\big)\\
&=\frac{1}{2}\log\det\big(\tilde{\eta}(s_t,a_t)^{\top}\tilde{\Lambda}_{t}^{-1}\tilde{\eta}(s_t,a_t)+\mathbf{I}\big),
\end{split}
\end{equation}
where the last line follows Matrix Determinant Lemma. Then, plugging \eqref{eq::tildelambda} into \eqref{eq::info-w}, we have 
\begin{equation}
\begin{split}
\tilde{\eta}(s_t,a_t)^{\top}\tilde{\Lambda}_{t}^{-1}\tilde{\eta}(s_t,a_t)&=
\begin{bmatrix}
\begin{smallmatrix}
\eta(t)^{\top} & 0      & \cdots & 0     \\
 0     & \eta(t)^{\top} & \cdots & 0         \\
\vdots & \vdots         & \ddots & \vdots\\
0      & 0              & \cdots & \eta(t)^{\top}
\end{smallmatrix}
\end{bmatrix}
\begin{bmatrix}
\begin{smallmatrix}
\Lambda_t^{-1} & 0 & \cdots & 0     \\
 0        & \Lambda_t^{-1} & \cdots & 0     \\
\vdots    & \vdots & \ddots & \vdots\\
0         & 0 & \cdots & \Lambda_t^{-1} 
\end{smallmatrix}
\end{bmatrix}
\begin{bmatrix}
\begin{smallmatrix}
\eta(t) & 0      & \cdots & 0     \\
 0     & \eta(t) & \cdots & 0         \\
\vdots & \vdots         & \ddots & \vdots\\
0      & 0              & \cdots & \eta(t)
\end{smallmatrix}
\end{bmatrix}\\&=
\begin{bmatrix}
\begin{smallmatrix}
\eta(t)^{\top}\Lambda_t^{-1}\eta(t) & 0      & \cdots & 0     \\
 0     & \eta(t)^{\top}\Lambda_t^{-1}\eta(t) & \cdots & 0         \\
\vdots & \vdots         & \ddots & \vdots\\
0      & 0              & \cdots & \eta(t)^{\top}\Lambda_t^{-1}\eta(t)
\end{smallmatrix}
\end{bmatrix}\in\mathbb{R}^{c\times c}.
\end{split}
\end{equation}
Thus, we have
\begin{equation}\label{eq::deteta}
\begin{split}
I({\rm vec}(W_t);[s_t,a_t,S_{t+1}]|\mathcal{D}_m)&=\frac{1}{2}\log\det\big(\tilde{\eta}(s_t,a_t)^{\top}\tilde{\Lambda}_{t}^{-1}\tilde{\eta}(s_t,a_t)+\mathbf{I}\big)\\
&=\frac{c}{2}\cdot\log\big(\eta(t)^{\top}\Lambda_t^{-1}\eta(t)+1\big).
\end{split}
\end{equation}
By assuming the L2-norm of feature vector $\|\eta\|_2\leq 1$ and $\lambda=1$ \citep{jin-2019,wang2020provably}, we have $\eta(t)^{\top}\Lambda_t^{-1}\eta(t)\leq 1$. Moreover, since $\frac{x}{2}\leq\log(1+x) \leq x$ for $x\in[0, 1]$, we have
\begin{equation}\label{eq:lsvi-tow-bonud}
\frac{\eta(t)^{\top}\Lambda_t^{-1}\eta(t)}{2} \leq \log \big(\eta(t)^{\top}\Lambda_t^{-1}\eta(t)+1\big)\leq \eta(t)^{\top}\Lambda_t^{-1}\eta(t).
\end{equation}
Therefore, we have
\begin{equation}\label{eq:lsvi-tow-bonud-2}
\sqrt{\frac{c}{4}}\:\big[\eta(t)^{\top}\Lambda_t^{-1}\eta(t)\big]^{\nicefrac{1}{2}} \leq I({\rm vec}(W_t);[s_t,a_t, S_{t+1}]|\mathcal{D}_m)^{\nicefrac{1}{2}} \leq \sqrt{\frac{c}{2}}\:\big[\eta(t)^{\top}\Lambda_t^{-1}\eta(t)\big]^{\nicefrac{1}{2}}.
\end{equation}

Moreover, since ${\rm vec}(W_t)$ is the vectorization of $W_t$, we have $I(W_t;[S_t,A_t]|\mathcal{D}_m)$=$I({\rm vec}(W_t);[S_t,A_t]|\mathcal{D}_m)$. Finally, we have
\begin{equation}
\sqrt{\frac{c}{4}}\:\big[\eta(t)^{\top}\Lambda_t^{-1}\eta(t)\big]^{\nicefrac{1}{2}} \leq I(W_t;[s_t,a_t, S_{t+1}]|\mathcal{D}_m)^{\nicefrac{1}{2}} \leq \sqrt{\frac{c}{2}}\:\big[\eta(t)^{\top}\Lambda_t^{-1}\eta(t)\big]^{\nicefrac{1}{2}}.
\end{equation}



Recall that we set $r^{\rm ucb}=\beta \big[\eta(t)^{\top}\Lambda_t^{-1}\eta(t)\big]^{\nicefrac{1}{2}}$. Thus, we obtain that 
\begin{equation}
\frac{1}{\sqrt{2}}\beta_0\cdot r^{\rm ucb} \leq I(W_t;[s_t,a_t,S_{t+1}]|\mathcal{D}_m)^{\nicefrac{1}{2}} \leq \beta_0\cdot r^{\rm ucb},
\end{equation}
where $\beta_0=\sqrt{\frac{c}{2\beta^2}}$ is a tuning parameter. 
Thus, we complete the proof of Theorem \ref{thm:ucb-informal}. 
\end{proof}

\subsection{Proof of Theorem \ref{thm:count-informal}}\label{app:proof-count}

\begin{proof}

In the sequel, we consider the tabular setting with finite state and action spaces. We define $d=|\mathcal{S}|\times|\mathcal{A}|$. In the tabular setting, we define the state-action encoding by the one-hot vector indexed by state-action pair $(s,a)\in \mathcal{S}\times\mathcal{A}$. For a state-action pair $(s_j,a_j) \in\mathcal{S}\times\mathcal{A}$, where $j\in [0,d-1]$, we have 
\begin{equation}
\eta(s_j,a_j)=
\begin{bmatrix}
0\\
\vdots\\
1\\
\vdots\\
0
\end{bmatrix}
\in \mathbb{R}^d,
\qquad
\eta(s_j,a_j)\eta(s_j,a_j)^{\top}=
\begin{bmatrix}
0      & \cdots & 0 & \cdots & 0\\
\vdots & \ddots &   &        & \vdots \\
0      &        & 1 &        & 0\\
\vdots &        &   & \ddots & \vdots \\
0      & \cdots & 0 & \cdots & 0\\
\end{bmatrix}
\in \mathbb{R}^{d\times d},
\end{equation}
where the value is 1 at the $j$-th entry and $0$ elsewhere. Moreover, the gram matrix $\Lambda_j=\sum_{i=0}^{m}\eta(x_j^i,a_j^i)\eta(x_j^{i},a_j^{i})^\top+\lambda \cdot \mathrm{\mathbf{I}}$ is the sum of all the matrices $\eta(s_j,a_j)\eta(s_j,a_j)^{\top}$ corresponding to the batch $\mathcal{D}_m$. That said, we have
\begin{equation}
\small
\Lambda_j=
\begin{bmatrix}
n_0+\lambda    & 0      &        & \cdots              & 0      \\
0      & n_1+\lambda    &        & \cdots              & 0      \\
\vdots &        & \ddots &                     & \vdots \\
0      &        &        & \!\!\!\!\!\!\!\!n_j+\lambda & 0      \\
\vdots &        &        & \:\:\:\:\:\:\ddots  & \vdots \\
0      & \cdots &        & \cdots              & n_{d-1}+\lambda
\end{bmatrix},
\end{equation}
where the $j$-th diagonal element of $\Lambda_j$ is the corresponding counts for state-action $(s_j,a_j)$, i.e., 
\[n_j=N_{s_j,a_j}.\] 
Thus, following the proof from \S\ref{app:proof-ucb}, the mutual-information scales with $\log \big(\eta(t)^{\top}\Lambda_t^{-1}\eta(t)+1\big)$. When the count $N_{s_j,a_j}$ of state-action pairs are large for all $j\in[0, d-1]$, we have $\log \big(\eta(t)^{\top}\Lambda_t^{-1}\eta(t)+1\big)\approx \eta(t)^{\top}\Lambda_t^{-1}\eta(t)$. Moreover, it holds that 
\begin{equation}
\big[\eta(s_j,a_j)^{\top}\Lambda_t^{-1}\eta(s_j,a_j)\big]^{\nicefrac{1}{2}}=\frac{1}{\sqrt{N_{s_j,a_j}+\lambda}}.
\end{equation}
Thus, in conclusion, we obtain that
\begin{equation}
\begin{split}
r^{\rm db}(s_j,a_j)= I\big(W_t;&[s_j,a_j,S_{j+1}]|\mathcal{D}_m\big) \approx \sqrt{\frac{c}{2}} \big[\eta(s_j,a_j)^{\top}\Lambda_t^{-1}\eta(s_j,a_j)\big]^{\nicefrac{1}{2}}\\
&=\sqrt{\frac{c}{2}}\frac{1}{\sqrt{N_{s_j,a_j}+\lambda}}=\beta_0 \cdot r^{\rm count}_j,
\end{split}
\end{equation}
where $c=|Z|$ is the same as the number of states $|\mathcal{S}|$ in tabular MDPs. Thus, we complete the proof of Theorem \ref{thm:count-informal}.
The bonus become smaller when the corresponding state-actions are visited more frequent, which follows the principle of count-based exploration \citep{count-2016,count-2017}. 
\end{proof}

\clearpage

\section{Implementation Detail}\label{app:db-detail}

\begin{table}[h!]
\small
  \caption{Experimental setup of PPO. The implementation of PPO is the same for all methods in our experiments.}
  \vspace{0.3em}
  \label{tab:hyper-ppo}
  \centering
  \begin{tabular}{p{0.15\columnwidth}p{0.14\columnwidth}p{0.63\columnwidth}}
    \toprule
    Hyperparameters & Value     & Description \\
    \midrule
    state space  & $84\times84\times4$ & Stacking 4 recent frames as the input to network.\\
	action repeat & 4 & Repeating each action 4 times. \\
    actor-critic network & conv(32,8,4) conv(64,4,2) conv(64,3,1) dense$\{512,512\}$ dense$|\mathcal{A}|+1$ & Using convolution(channels, kernel size, stride) layers first, then feed into two fully-connected layers, each with 512 hidden units. The outputs of the policy and the value function are split in the last layer. The output of policy is a Softmax function. The output of the value function is a single linear unit. \\
    entropy regularizer & $10^{-3}$ &  The loss function of PPO includes an entropy regularizer of the policy to prevent premature convergence\\
    $\gamma_{\rm ppo}$ & 0.99 & Discount factor of PPO \\
    $\lambda_{\rm ppo}$ & 0.95 & GAE parameters of PPO \\
	normalization & mean of 0, std of 1 & Normalizing the intrinsic reward and advantage function by following \citep{curiosity-2017,largescale-2019}. The advantage estimations in a batch is normalized to achieve a mean of 0 and std of 1. The intrinsic rewards is smoothen exponentially by dividing a running estimate of std. \\
    learning starts & 50000 & The agent takes random actions before learning starts. \\
    replay buffer size & 1M & The number of recent transitions stored in the replay buffer. \\
    training batches & 3 & The number of training batches after interacting an episode for all actors. \\
    optimizer & Adam & Adam optimizer is used for training. Detailed parameters: $\beta_1=0.9$, $\beta_2=0.999$, $\epsilon_{\rm ADAM}=10^{-7}$.\\
    mini-batch size & 32 & The number of training cases for gradient decent each time. \\
    learning rate & $10^{-4}$ & Learning rate for Adam optimizer. \\
	max no-ops & 30 & Maximum number no-op actions before an episode starts. \\
	actor number & 128 & Number of parallel actors to gather experiences.  \\
	$T$ & 128 & Episode length. \\
    \bottomrule
  \end{tabular}
\end{table}

\begin{table}[h!]
\small
  \caption{Experimental setup of DB}
  \vspace{0.3em}
  \label{tab:hyper-db}
  \centering
  \begin{tabular}{lll}
    \toprule
    Hyperparameters & Value     & Description \\
    \midrule
    network architecture  & Appendix \ref{app:db-train} & See Fig.\ref{fig:db-network}.\\
	$\alpha_1$ & \{0.1, 0.001\} & Factor for $I_{\rm upper}$. \\
	$\alpha_2$ & 0.1 & Factor for $I_{\rm pred}$. \\
	$\alpha_3$ & \{0.1, 0.01\} & Factor for $I_{\rm nce}$. \\
	$\tau$     & 0.999 & Factor for momentum update. \\
	\bottomrule
  \end{tabular}
\end{table}

\begin{table}[h!]
\small
\centering
\caption{Comparison of model complexity. (1) ICM estimates the inverse dynamics for feature extraction with 2.21M parameters. ICM also includes a dynamics model with 2.65M parameters to generate intrinsic rewards. (2) Disagreement use a fixed CNN for feature extraction thus the trainable parameters is 0. Disagreement uses an ensemble of dynamics with total 26.47M parameters to estimate the uncertainty. (3) CB does not requires any additional parameters compared to the actor-critic network. (4) DB requires slightly more parameters than ICM in representation learning, while do not uses additional parameters in estimating the dynamics since the DB-bonus is directly derived from the information gain of latent representation.}
\vspace{0.3em}
\label{tab:hyper-complexity}
\begin{tabular}{l|rr|r}
    \hline
    {~} & Feature extractor & Dynamic model & Total\cr
    \hline
    ICM          & 2.21M  & 2.65M  &  4.86M          \\
    Disagreement & 0M     & 26.47M &  26.47M         \\
    CB           & 0M     & 0M     &  0M             \\
    \textbf{DB (ours)}    & 5.15M  &  0M &  5.15M \\    
\hline
\end{tabular}
\end{table}

\clearpage
\section{Supplementary Experimental Results}\label{app:experiment}

\subsection{Random-Box Noise}\label{app:experiment-box}

We present the evaluation curves of Atari games with \emph{random-box noise} in Fig.~\ref{fig:result-atari-box}. 

~~~~\\~~\\~~\\~~\\
\begin{figure}[h]
\centering
\includegraphics[width=5.5in]{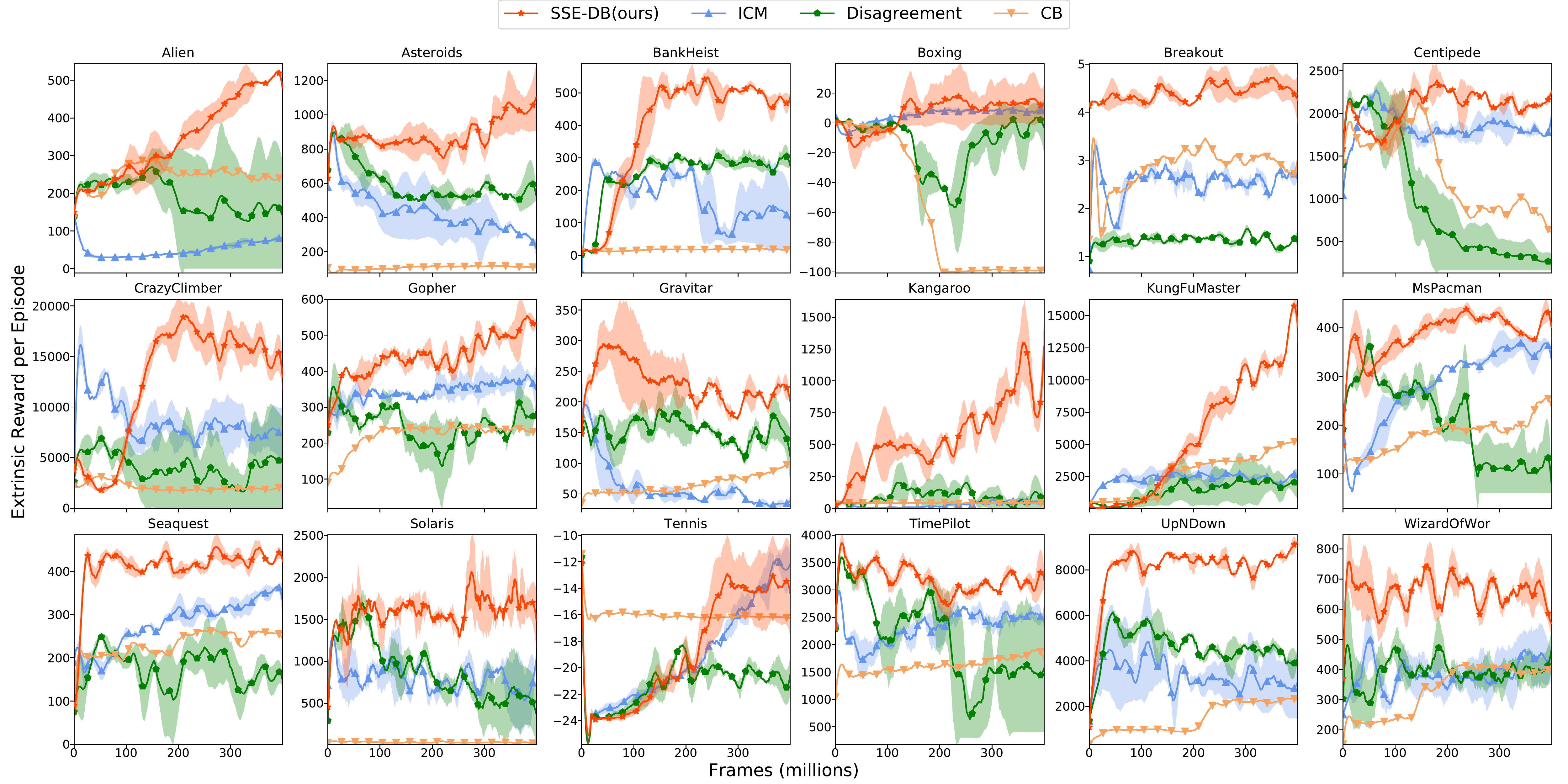}
\caption{Evaluation curves in Atari games with \emph{random-box noise}. SSE-DB outperforms all the baselines in 17 out of the 18 tasks. Comparing to standard Atari, we observe that the performance of SSE-DB with random-box noise is suboptimal in Breakout, Gopher, and WizardOfWar. In these tasks, dynamic-relevant information (e.g., the ball in Breakout, the tunnels in Gopher, and the worriors in WizardOfWar) are (partly) masked by random-boxes. Breakout is affected by random boxes significantly as the ball is too small to be distinguished from the noise. In addition, we observe that ICM is prone to random-box noise in most of the tasks. Disagreement also demonstrates decreased (e.g., CrazyClimber, Gopher, Kangaroo) or unstable (e.g., Alien, Centipede, TimePilot) performance in most of the tasks. A possible explanation is that the random-box noise affects the training of the dynamics models, thus bring adverse impact in estimating the Bayesian uncertainty through ensembles in Disagreement. In contrast, SSE-DB performs well in the presence of the random-box noise.}
\label{fig:result-atari-box}
\end{figure}

\clearpage

\subsection{Pixel Noise}\label{app:experiment-pixel}

We present the evaluation curves in Atari games with \emph{pixel noise} in Fig.~\ref{fig:result-atari-pixel} and Fig.~\ref{fig:compare-atari-pixel}. 

\vspace{3em}

\begin{figure}[h!]
\centering
\includegraphics[width=5.5in]{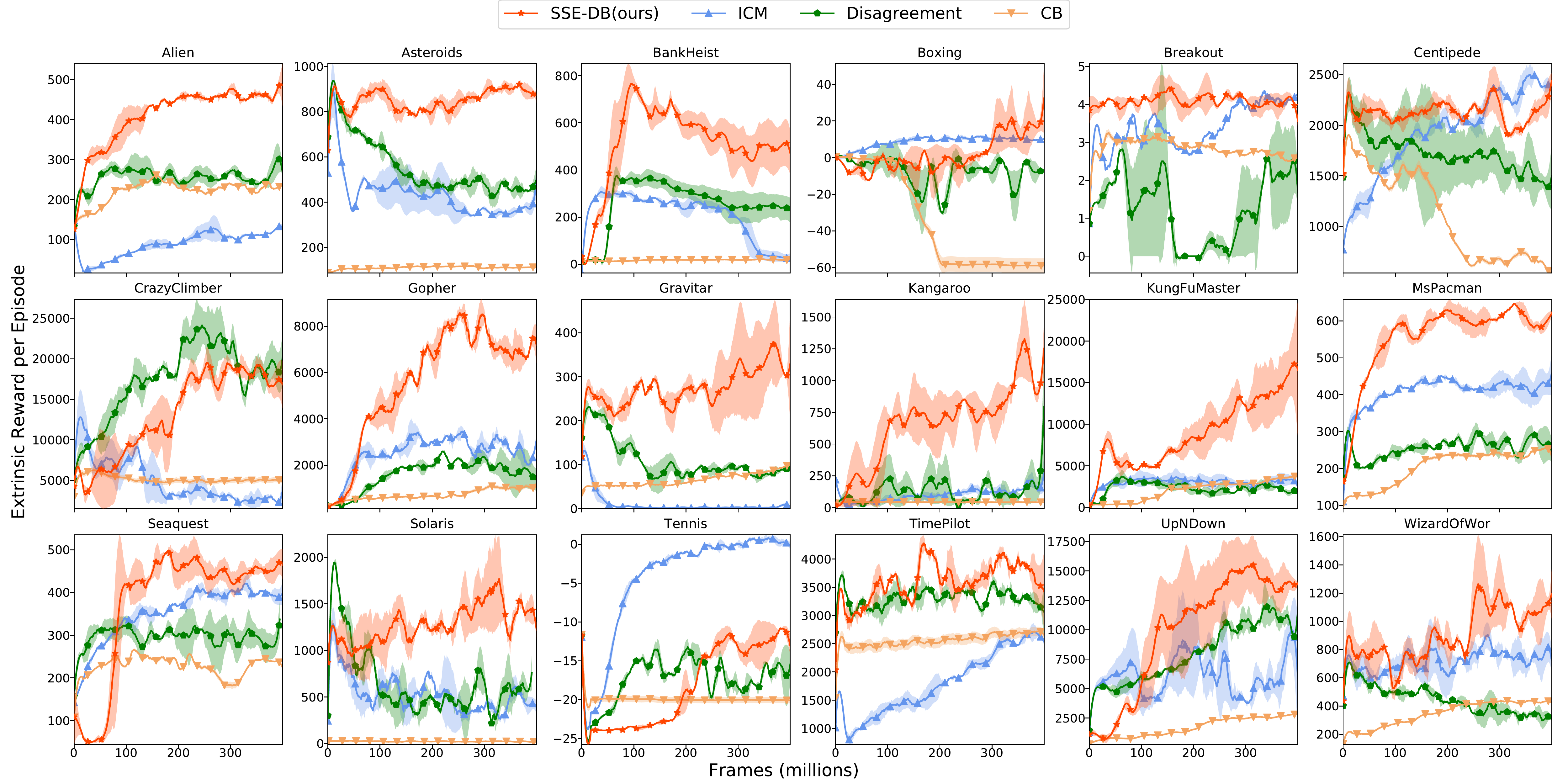}
\caption{Evaluation curves in Atari games with \emph{pixel noise}. SSE-DB outperforms all the baselines in 15 out of the 18 tasks. For games with the pixel noise, the performance of SSE-DB is similar to that for standard Atari games, expect for Breakout, in which the ball is too small to be distinguished from the pixel noise. We refer to Fig.~\ref{fig:compare-atari-pixel} for a comparison of results with and without pixel noise.}
\label{fig:result-atari-pixel}
\end{figure}

\vspace{3em}

\begin{figure}[h!]
\centering
\includegraphics[width=5.5in]{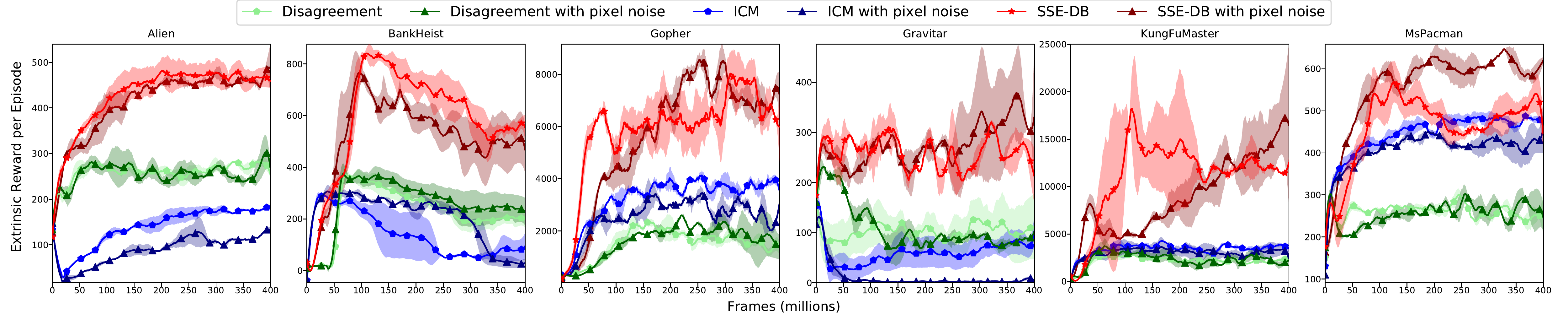}
\caption{A comparison of results on selected Atari games with and without \emph{pixel noise}. SSE-DB shows robustness to pixel noise since we discard the dynamics-irrelevant information. Nevertheless, we find that the adverse effect of introducing pixel noise is limited for exploration, both in SSE-DB and other baselines. Especially, in \emph{Gravitar} and \emph{MsPacman}, the performance of SSE-DB has sight improvement compared to that of standard Atari games. A possible explanation is that introducing the pixel noise leads to data augmentation for image inputs, thus bringing regularization for learning and enhance the generalization ability of the policy.}
\label{fig:compare-atari-pixel}
\end{figure}

\clearpage
\subsection{Sticky Actions}\label{app:experiment-sticky}

Evaluation curves in Atari games with \emph{sticky actions} is give in Fig.~\ref{fig:result-atari-sticky}.

~~~~\\~~\\~~\\~~\\

\begin{figure}[h!]
\centering
\includegraphics[width=5.5in]{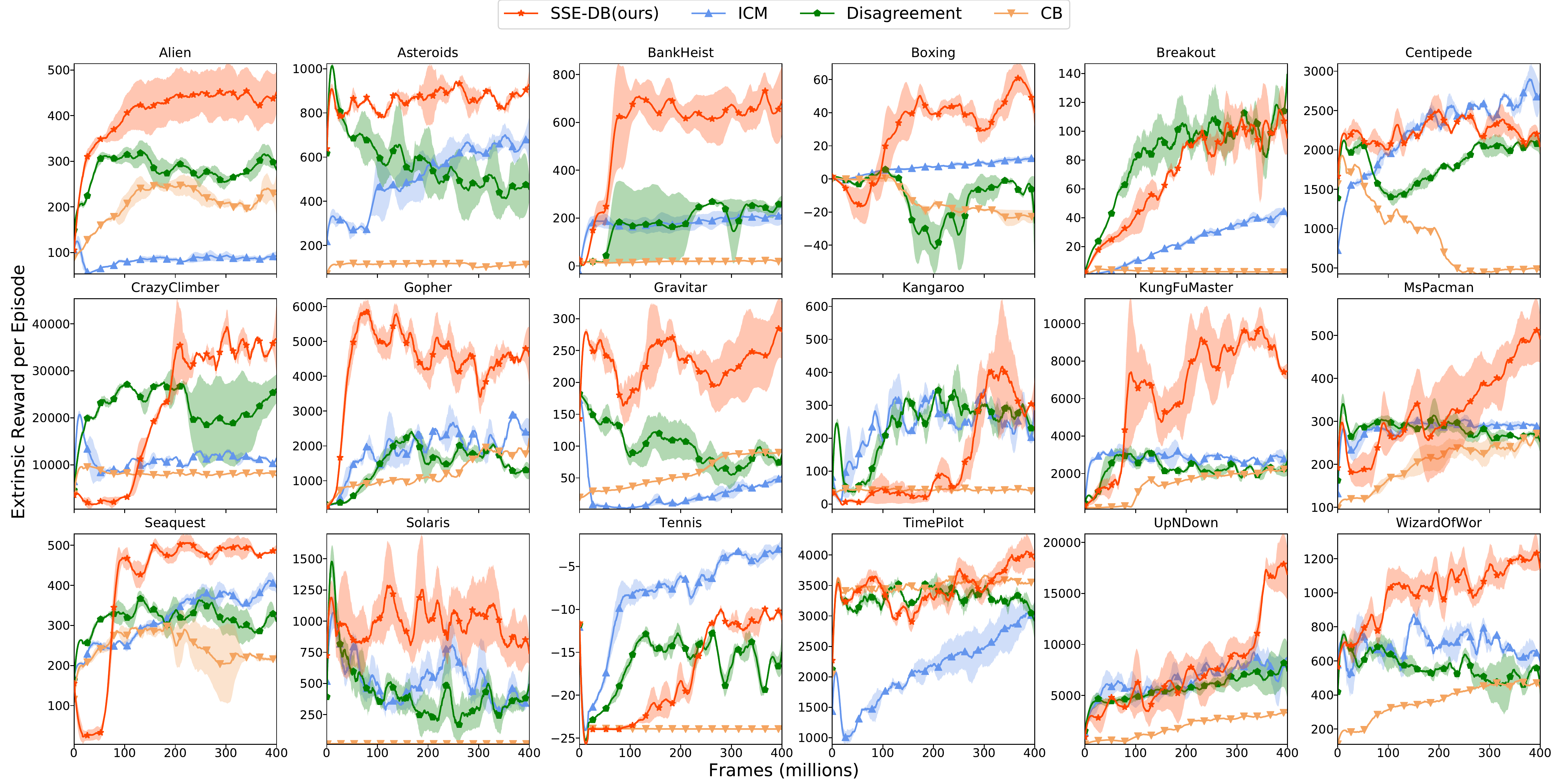}
\caption{Evaluation curves in Atari games with \emph{sticky actions}, which introduces additional stochasticity to the transition dynamics. As shown in figures, SSE-DB outperforms all baselines in 15 out of 18 games. We observe that in Breakout, MsPacman, and UpNDown, SSE-DB performs suboptimal in the early stage of training stage. Nevertheless, the performance gradually improves as the training evolves and reaches to the optimum eventually, which occurs since DB removes the noisy information for the state representation along the training process. We highlight that SSE-DB compresses the latent representation of the corresponding state-action pair, thus minimizing the effect of noisy actions. In addition, Disagreement method based on Bayesian uncertainty also demonstrates robustness to sticky actions, despite the fact that it needs more parameters than other methods according to Table \ref{tab:hyper-complexity}.}
\label{fig:result-atari-sticky}
\end{figure}

\clearpage
\subsection{Ablation Study}\label{app:experiment-abla}

In this section, we present several ablation studies to analyze the components of the DB model. We conduct experiments on different settings that remove different components in DB, namely,
\begin{itemize}[leftmargin=13pt]
\item \emph{No-Upper}, which is a variant of DB that removes $I_{\rm upper}$ from $\mathcal{L}_{\rm DB}$ (set $\alpha_1=0$) in \eqref{eq:db-loss};
\item \emph{No-Pred}, which is a variant of DB that removes $I_{\rm pred}$ from $\mathcal{L}_{\rm DB}$ (set $\alpha_2=0$);
\item \emph{No-NCE}, which is a variant of DB that removes $I_{\rm nce}$ from $\mathcal{L}_{\rm DB}$ (set $\alpha_3=0$); and
\item \emph{No-NCE-Momentum}, which is a variant of DB that removes $I_{\rm nce}$ and utilizes the same encoder $f^S(;\theta)$ for successive observations $o_t$ and $o_{t+1}$, in contrast with the momentum observation encoding for $o_{t+1}$ of DB.
\end{itemize}
We conduct experiments on the Alien task with standard observation and random-Box noise. We illustrate the results in Fig.~\ref{fig:ablation} and Fig.~\ref{fig:ablation-box}, respectively. We illustrate the performance comparison on extrinsic rewards in (a) of Fig.~\ref{fig:ablation}, and the change of $I_{\rm upper}$, $I_{\rm pred}$ $I_{\rm nce}$ in (b), (c) and (d), respectively, of Fig.~\ref{fig:ablation}. We discuss the results in the sequel.

\begin{itemize}[leftmargin=13pt]
\item \emph{No-Upper} setting exhibits similar performance as SSE-DB in standard Atari. Without compressing the representation through minimizing $I_{\rm upper}$, the latent $Z$ preserves more information and exhibits well exploration with DB-bonus. However, in random-box setting, the latent $Z$ contains distractors features thus bringing adverse effects in exploration. Interestingly, $I_{\rm upper}$ first increases and then decrease without minimizing the $I_{\rm upper}$ explicitly (see Fig.~\ref{fig:ablation} (b) and Fig.~\ref{fig:ablation-box} (b)). This phenomenon is reminiscent of previous studies of Information Bottleneck in Deep Learning \citep{ibn-2015,ibn-2017}, suggesting that the neural network has the ability to actively compress the input for efficient representation. Nevertheless, according to our experiments, such compression is still not sufficient to handle the random-box noise in exploration. In such environments, the objective $I_{\rm upper}$ needs to be minimized to discard dynamics-irrelevant features.
\item \emph{No-Pred} setting has significantly reduced performance and stability for both the standard observations and the observations with random-box noise, comparing with SSE-DB. According to our results, the contrastive estimation cannot replace the role of predictive objective in maximizing the mutual information of $Z_t$ and $S_{t+1}$. A possible explanation is that the contrastive loss is applied in a transformed space through projection heads, thus part of the dynamics-relevant information is encoded in the projection encoding but not the latent representation $Z$. However, our DB-bonus is defined in $Z$, thus do not capture the information encoded in projection heads. Besides, according to Fig.~\ref{fig:ablation} (c) and Fig.~\ref{fig:ablation-box} (c), the predictive objective does not increase through solely maximizing $I_{\rm nce}$. 
\item \emph{No-NCE} setting exhibits similar performance as DB in standard Atari. Since we use the momentum observation encoder to prevent collapsing solution, the lack of contrastive estimation does not bring significant performance loss. Meanwhile, according to Fig.~\ref{fig:ablation} (d) and and Fig.~\ref{fig:ablation-box} (d), we find that $I_{\rm nce}$ improves slightly without being optimized explicitly, since the predictive objective helps improving $I_{\rm nce}$ implicitly. Nevertheless, the lack of NCE-loss produces reduced performance with random-box noise. Our experiments show that using contrastive estimation leads to a stronger distillation for the dynamics-relevant feature of observations.
\item \emph{No-NCE-No-Momentum} setting leads to the collapsing solution in representation. According to Fig.~\ref{fig:ablation} (c) and and Fig.~\ref{fig:ablation-box} (c), the predictive objective converges very fast by encoding all observations to uninformative values. Such a trivial solution is undesirable as it does not capture any meaningful information. The performance of such setting decrease significantly comparing against SSE-DB.
\end{itemize}

\begin{figure}[h!]
\centering
\includegraphics[width=5.4in]{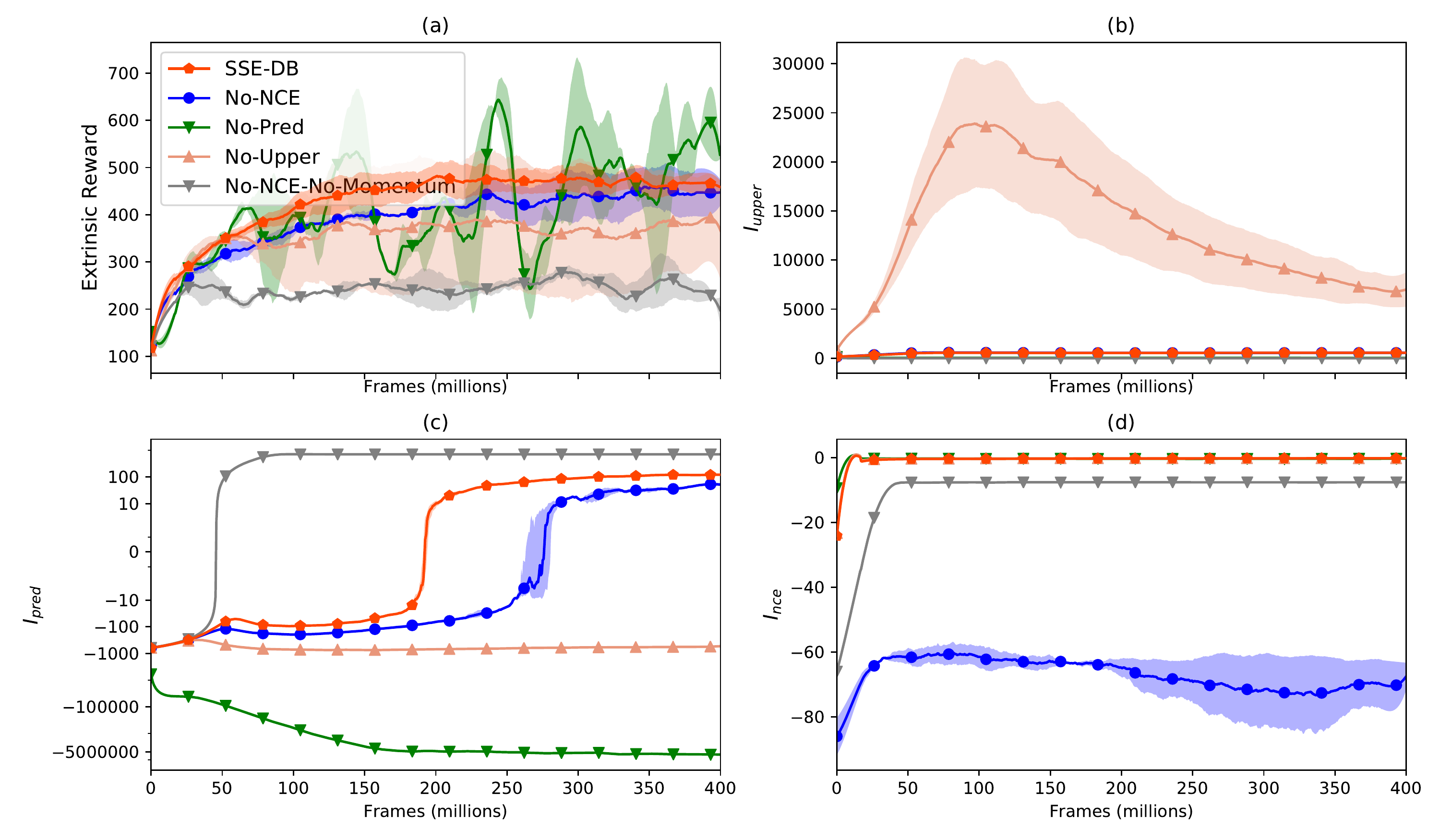}
\caption{Ablation study in \emph{Alien}. We measure (a) extrinsic rewards, (b) $I_{\rm upper}$ that indicates the amount of information contained in the representation space, (c) the predictive objective $I_{\rm pred}$, and (d) the contrastive estimation $I_{\rm nce}$ for comparison.}
\label{fig:ablation}
\end{figure}

\begin{figure}[h!]
\centering
\includegraphics[width=5.4in]{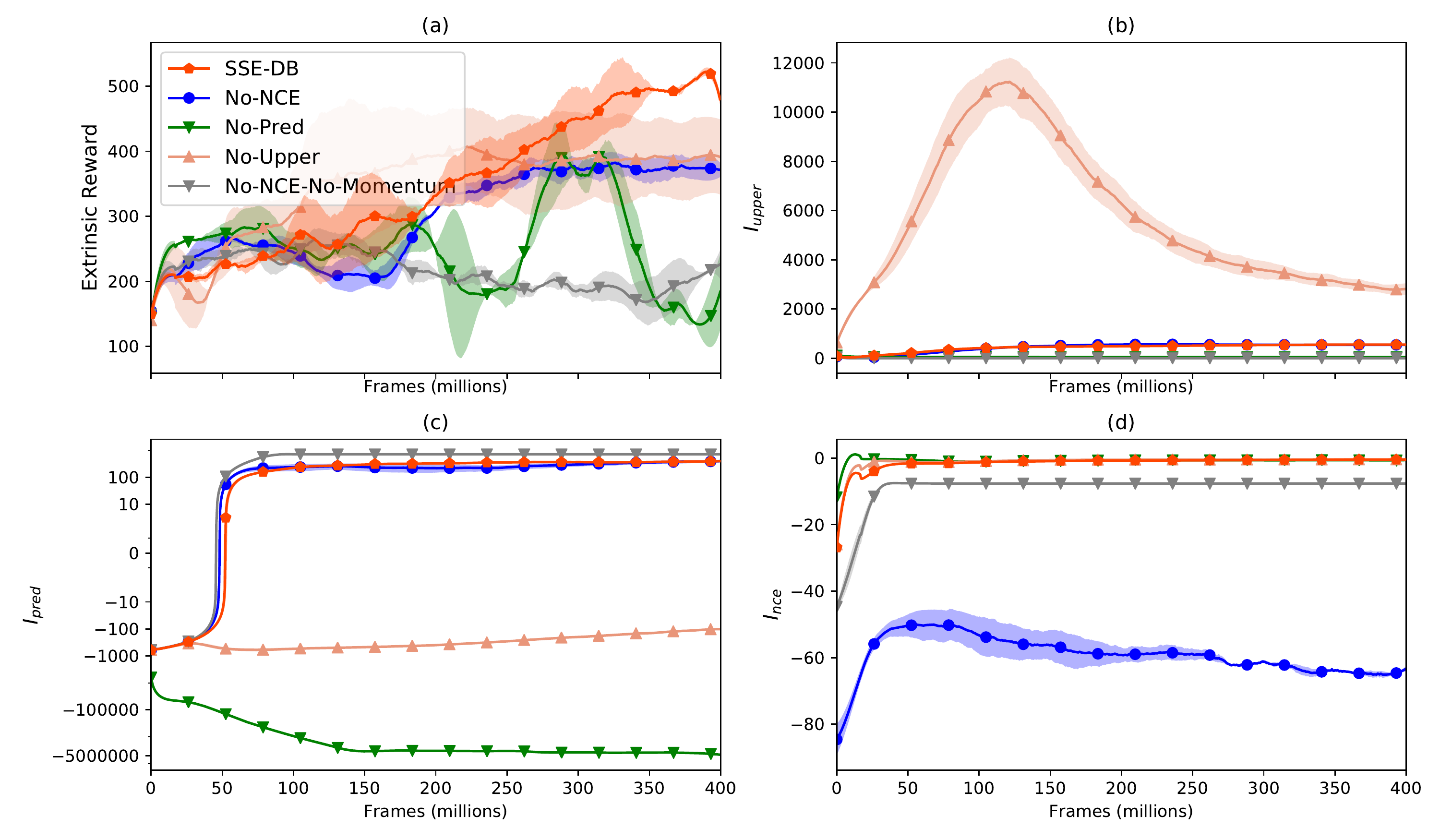}
\caption{Ablation study in \emph{Alien} with \emph{random-box noise}. We measure (a) extrinsic rewards, (b) $I_{\rm upper}$ that indicates the amount of information contained in the representation space, (c) the predictive objective $I_{\rm pred}$, and (d) the contrastive estimation $I_{\rm nce}$ for comparison.}
\label{fig:ablation-box}
\end{figure}

\clearpage
\subsection{Visualizing DB-bonus}\label{app:vis-bonus}

The proposed DB-bonus motivates the agent to explore states and actions that have high information gain to the representation. To further understand the DB-bonus, we provide visualization in two tasks to illustrate the effect of DB-bonuses. We choose two Atari games Breakout and Gopher, and visualize the DB-bonus in an episode based on a trained DB model.

\subsubsection{Breakout}

In Breakout, the agent uses walls and the paddle to rebound the ball against the bricks and eliminate them. We use a trained SSE-DB agent to interact with the environment for an episode in Breakout. The whole episode contains 1942 steps, and we subsample them every 4 steps for visualization. The curve in Fig.~\ref{fig:rewards-breakout} shows the UCB-bonus in 481 sampled steps. 

We select $16$ spikes of the DB-bonus on the trajectory and then visualize their corresponding observations. From the results, we find that the spikes typically correspond to some critical observations, including eliminating bricks (e.g., frames $1$, $3$, and $5$-$8$), rebounding the ball (e.g., frames $2$ and $4$), digging a tunnel (e.g., frames $9$-$12$), and throwing the ball onto the top of bricks (e.g., frames $13$-$16$). These examples demonstrate that the DB-bonus indeed encourages the agent to explore many crucial states, even without knowing the extrinsic rewards. The DB-bonus encourages the agent to explore the potentially informative and novel state-action pairs to get high rewards. We also record 15 frames after each spike for further visualization. The video is available at \url{https://www.dropbox.com/sh/gw7m38o29dfl9zx/AACF1AogB93spuD_Vsk_lsOBa?dl=0}.

\begin{figure}[h!]
\vspace{1em}
\centering
\includegraphics[width=5.4in]{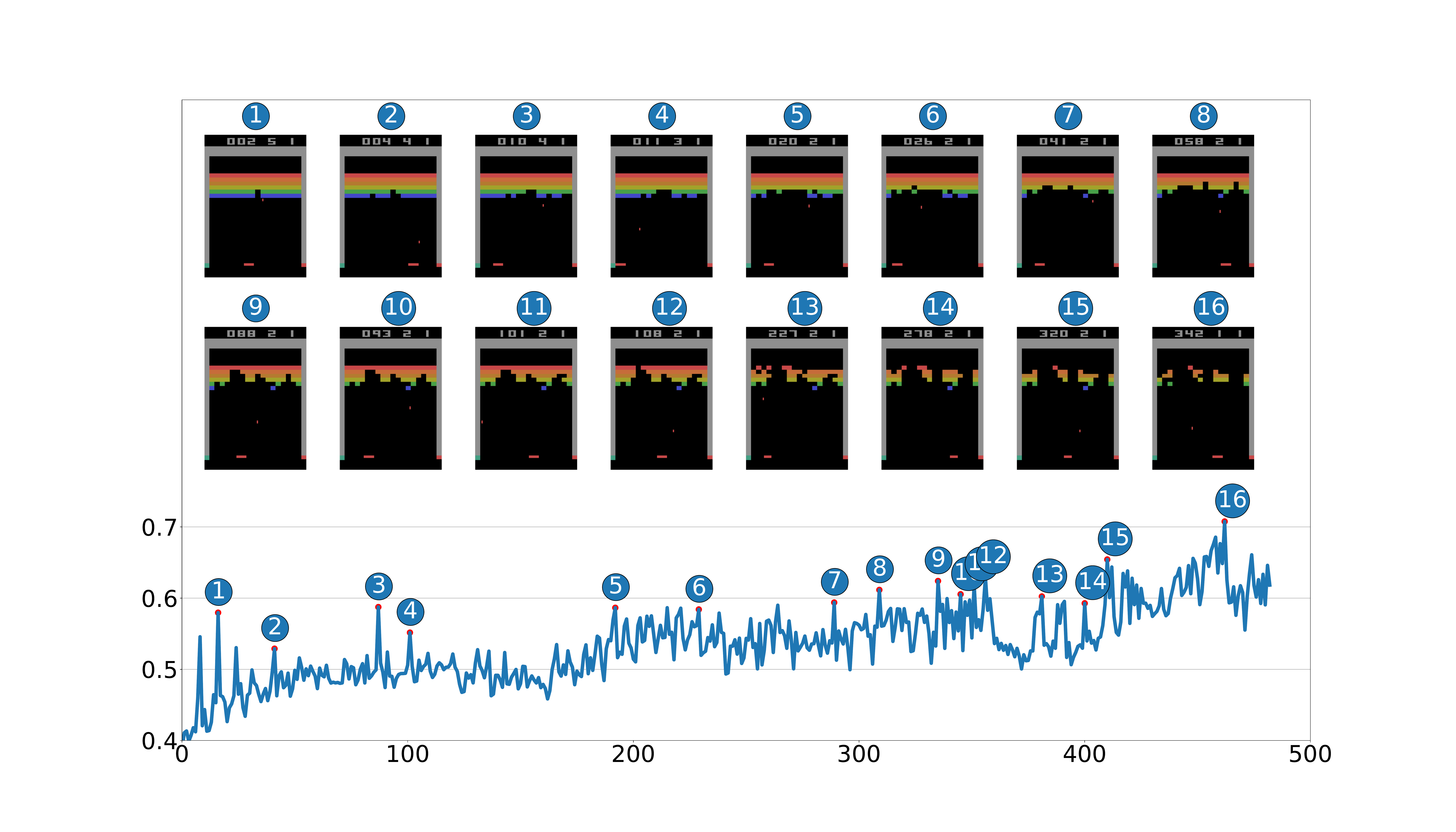}
\caption{Visualization of DB-bonus for an episode in Breakout. The curve corresponds to the DB-rewards of a trajectory. The numbers from $1$ to $16$ corresponds to the selected spikes, and the images are the corresponding observations.}
\label{fig:rewards-breakout}
\end{figure}

\subsubsection{Gopher}

Gopher is a popular Atari game. In this game, the gopher tunnels left, right, and up to the surface. When the gopher makes a hole to the surface, he will attempt to steal a carrot. If the holes have been tunneled to the surface, the farmer (i.e., the agent) can hit the gopher to send him back underground or fill in the holes to prevent him from reaching the surface. Rewards are given when the agent hits the holes and gopher. SSE-DB performs well in this task. To illustrate how DB-bonus works, we use an SSE-DB agent to play this game for an episode and records the DB-bonus in all 4501 steps. Fig.~\ref{fig:rewards-gopher} shows the DB-bonus and the corresponding frames in 16 chosen spikes in 1125 subsampled steps.

We find almost all spikes (i.e., frames 1-3, 5-13, 15-16) of DB-bonus correspond to scenarios that the gopher makes a hole to the surface, which is rarely occurs and signifies that the gopher will have a chance to eat carrots. Also, these scenarios are crucial for the farmer to get rewards since the farmer can hit the gopher and holes to prevent the carrots from being eaten. The DB-bonus encourages the farmer to learn skills to move fast and hit the holes in the surface to obtain high scores. In addition, the gopher moves underground in frames 4 and 14, and the farmer mends holes in the surface. The transitions with high DB-bonuses make the agent explore the environment efficiently. The video of spikes is available at \url{https://www.dropbox.com/sh/gw7m38o29dfl9zx/AACF1AogB93spuD_Vsk_lsOBa?dl=0}.

\begin{figure}[h!]
\centering
\includegraphics[width=5.4in]{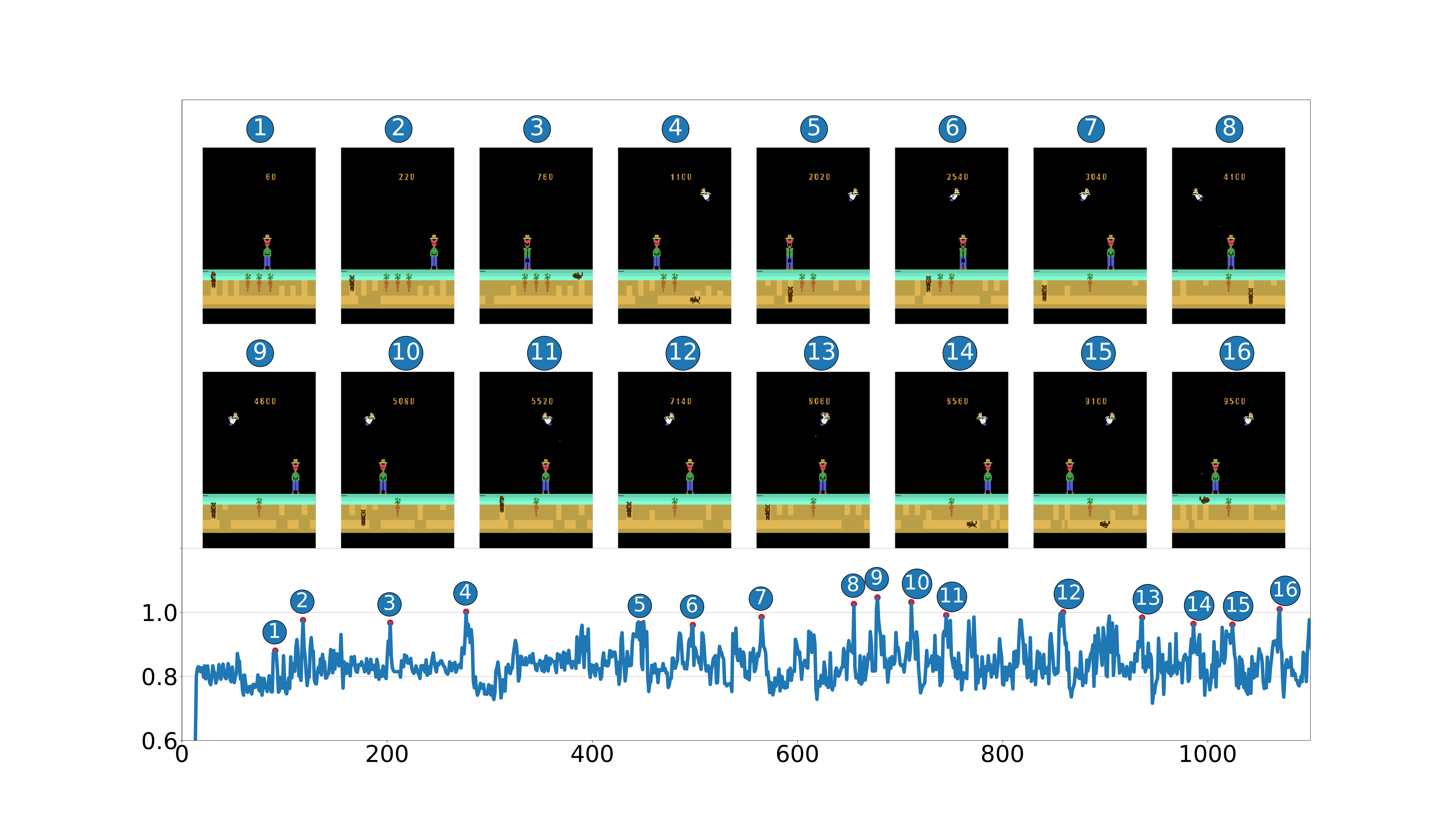}
\caption{Visualization of DB-bonus for an episode in \emph{Gopher}.}
\label{fig:rewards-gopher}
\end{figure}

\clearpage
\subsection{Montezuma's Revenge}

Several exploration methods demonstrate superior performance on Montezuma's Revenge \citep{bonus-2020}. However, these methods all use intrinsic rewards along with extrinsic rewards from the environment in training. In the self-supervised exploration setting where the training hinges solely on the intrinsic rewards, we find that SSE-DB and all the other baselines scores zero in this task.

Nevertheless, we observe that the agent trained with DB-bonus can pass half of the first room in Montezuma's Revenge. We conduct t-SNE \citep{tsne-2008} visualization to illustrate the learned representation of DB. Form Fig.~\ref{fig:montezuma-a}, we find that the latent representations are well aligned in several clusters, which corresponds to stepping down the ladder, jumping to the pillar, and escaping an enemy, respectively. We also visualize the raw states from the same episode with t-SNE in Fig.~\ref{fig:montezuma-b}. In contrast, we do not find any meaningful clusters form the visualization of raw states. DB enables to capture certain aspects of such a hard task. 

We give a video of the trained policy of DB \footnote{\url{https://www.dropbox.com/s/boijqmt66mgnj17/montezuma-DB.mp4?dl=0}}. From the video, we find that the self-supervised agent with DB-bonus can learn some skills, including stepping down the ladder, jumping to the pillar, and trying to escape the enemy. Nevertheless, as the learning curve suggests, such learned skills along are insufficient for obtaining scores.

\vspace{3em}

\begin{figure}[h!]
  \centering
  \subfigure[T-SNE visualization of the latent representation and the corresponding frames]{\includegraphics[width=3.5in]{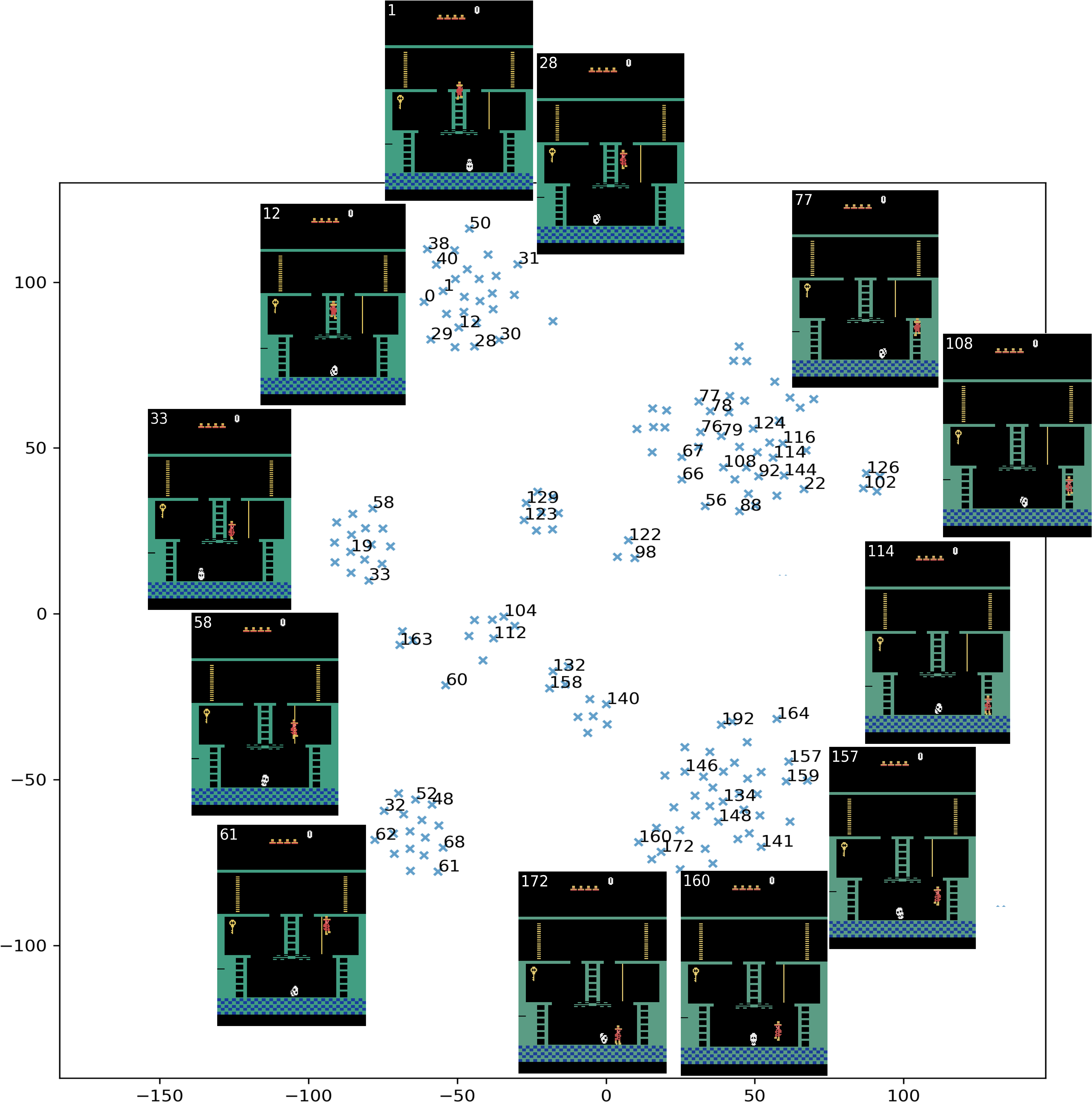}\label{fig:montezuma-a}}
  \hspace{1.0em}
  \subfigure[T-SNE visualization of raw states]{\includegraphics[width=1.5in,totalheight=2.5in]{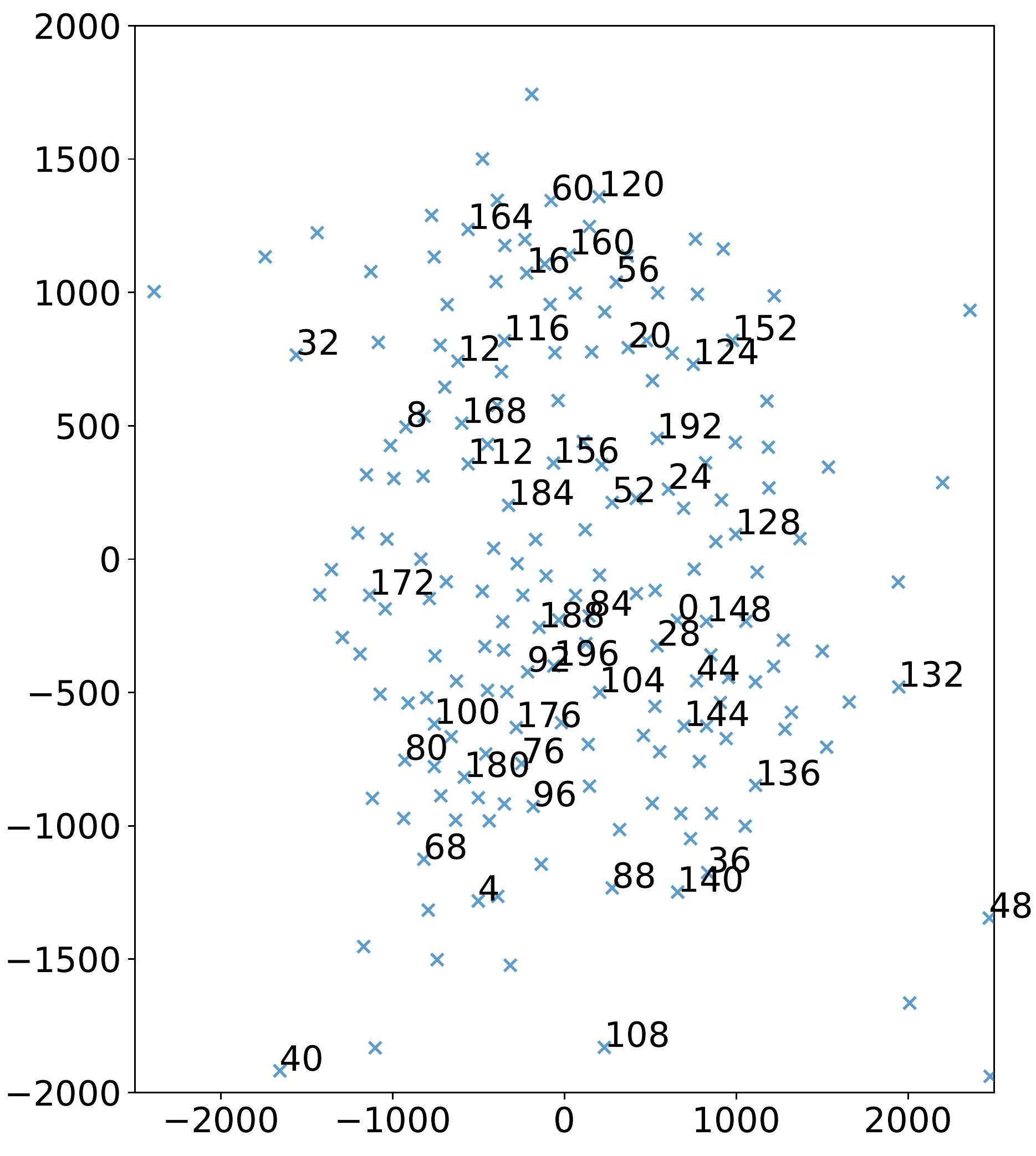}\label{fig:montezuma-b}}
  \caption{(a) Visualization for 200 observations of Montezuma's Revenge in an episode. We visualize the latent representations of DB in 2 dimensions with t-SNE \citep{tsne-2008}. Numbers on top-left of game frames correspond to numbers of representations in lower-dimensional space. (b) Visualization for raw states ($84\times84\times4$ for each one) in the same episode. }
  \label{fig:montezuma}
\end{figure}

\clearpage
\subsection{Comparison with Entropy-Based Exploration}

There exist several methods that perform entropy-based exploration for unsupervised representation learning, and then use this representation for downstream task adaptation, including VISR \citep{hansen2019fast}, APT \citep{liu2021behavior}, APS \citep{liu2021aps}, RE3 \citep{RE3-2021} and Proto \citep{Proto-2021}. 
The entropy-based methods use $k$-nearest neighbor state entropy estimator to estimate the entropy of state $\mathcal{H}(s)$ and then use it as intrinsic rewards.

In this section, we focus on the unsupervised exploration stage and compare SSE-DB with entropy-based exploration methods. Since APT and APS do not release code and ProtoRL conducts experiments in DeepMind control rather than Atari, we conduct experiments with RE3 algorithm. Since RE3 uses Rainbow as the basic algorithm, and uses both the extrinsic and intrinsic rewards in training, we re-implement the RE3 bonus in our codebase to evaluate its performance in a self-supervised setting with noisy environments. 
As shown in Fig.~\ref{fig:RE3-randombox}, RE3 performs reasonably in standard Atari games. However, the performance decreased significantly in the Random-Box Atari environments. A possible reason is that the entropy of the state increases significantly if we inject noises. Hence, exploration is misled by the noises in these environments. Nevertheless, as shown in Fig.~\ref{fig:RE3-sticky}, the entropy-based methods are robust to sticky actions since they use the entropy of states in exploration, without considering the entropy of actions.

\begin{figure}[h!]
\centering
\includegraphics[width=5.4in]{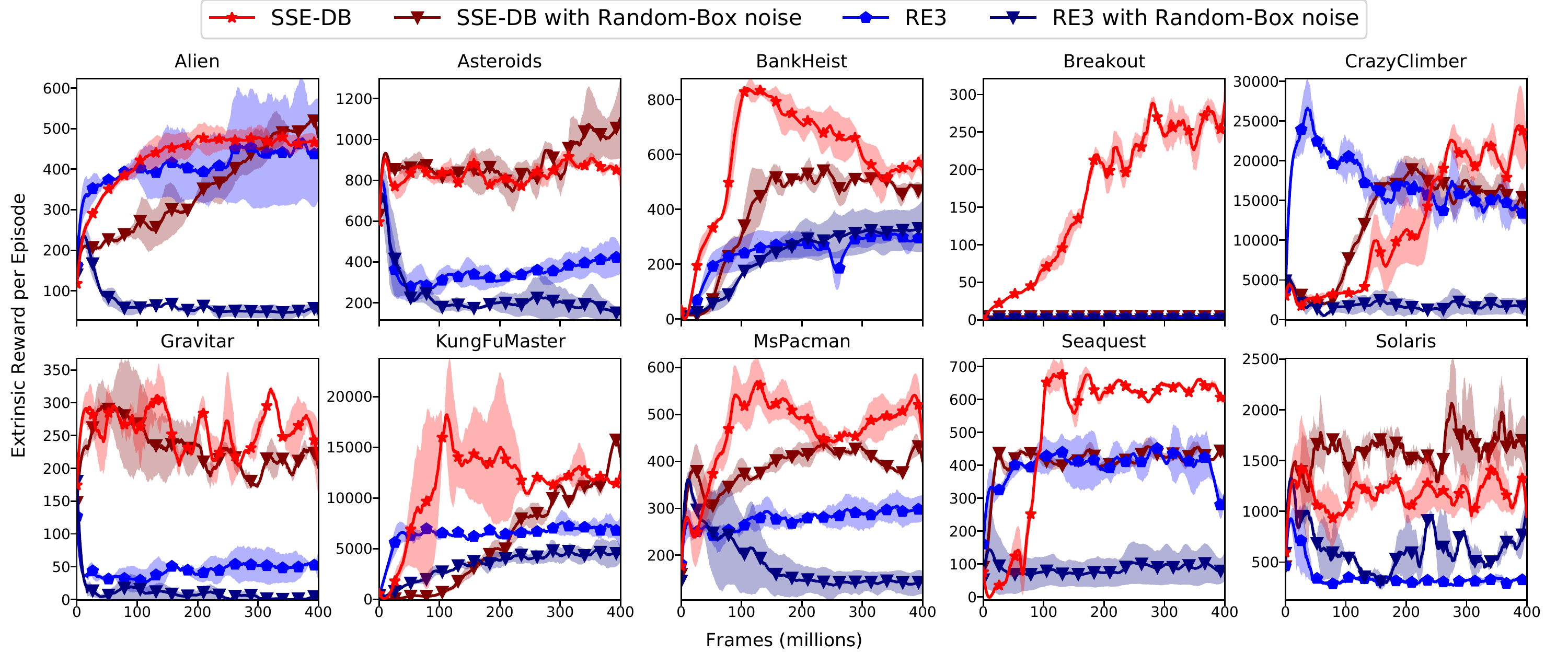}
\caption{A comparison of results on selected Atari games with and without Random-Box noise. SSE-DB shows robustness to Random-Box noise while RE3 is sensitive to the noises.}
\label{fig:RE3-randombox}
\end{figure}

\begin{figure}[h!]
\centering
\includegraphics[width=5.4in]{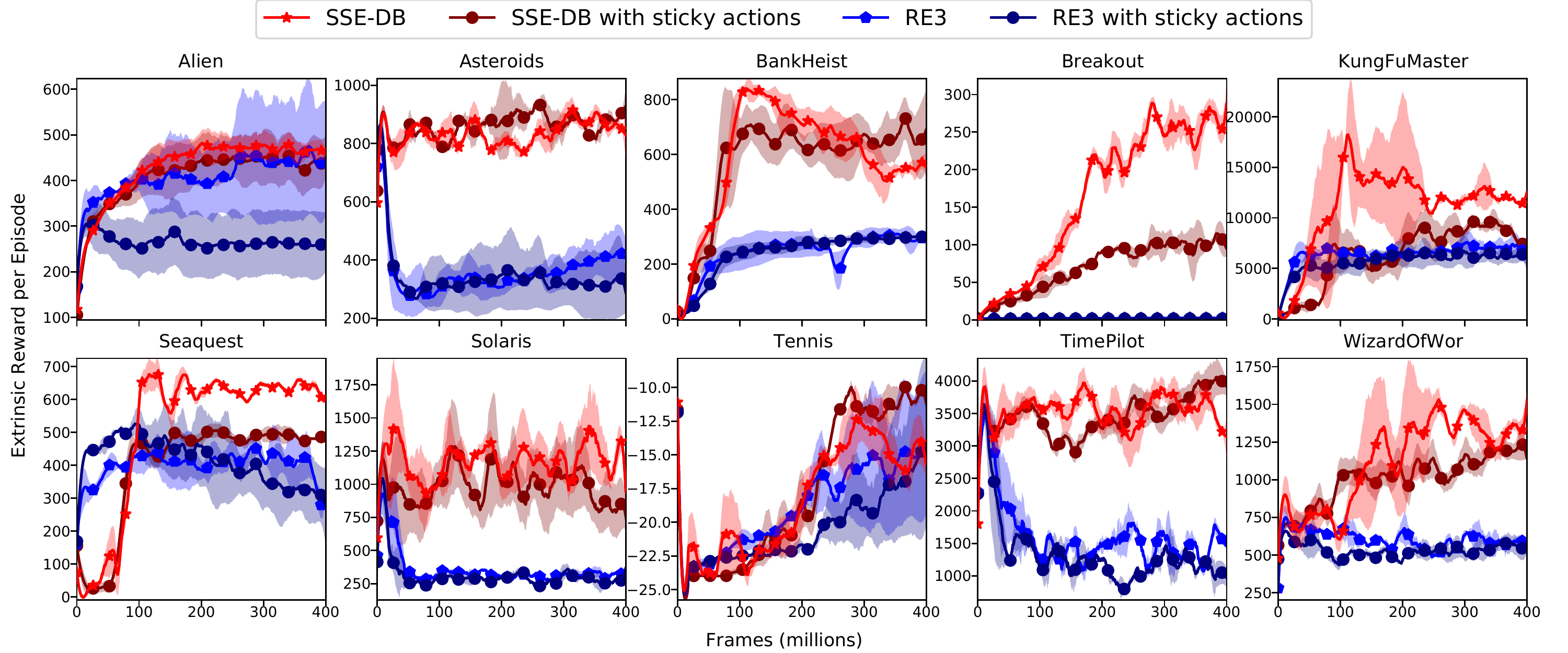}
\caption{A comparison of results on selected Atari games with and without sticky actions. Both SSE-DB and RE3 show robustness to sticky actions.}
\label{fig:RE3-sticky}
\end{figure}

\end{document}